\pdfoutput=1

\documentclass[10pt,onecolumn]{article}

\usepackage{setspace}
\usepackage{amsmath}
\usepackage{times}
\usepackage{amssymb}
\usepackage{verbatim}
\usepackage[vlined,ruled]{algorithm2e}
\usepackage{multirow}
\usepackage{subfigure}
\usepackage{graphicx,cs,mathsymb,xspace,colortbl} 
\usepackage{sidecap}
\usepackage{stmaryrd}
\usepackage{tikz}
\usepackage{bibspacing}
\usepackage[footnotesize]{caption}
\usepackage{float}

\def\iid{i.i.d\onedot} 
\def\x{{\mathbf x}}
\def\b{{\mathbf b}}
\def\B{{\mathbf B}}
\def\y{{\mathbf y}}
\def\mathr{{\mathbb{R}}}
\def\mathX{{\mathcal{X}}}

\def\loss{{\text{Loss}}}
\def\rank{{\text{Rank}}}
\def\ER{{\text{ER}}}
\def\reg{{\text{Reg}}}
\def\Beta{ {\boldsymbol{\beta}} }
\def\Theta{ {\boldsymbol{\theta}} }
\def\Varphi{ {\boldsymbol{\varphi}} }
\def\bxi{ {\boldsymbol{\xi}} }
\def\bomega{ {\boldsymbol{u}} }
\def\Alpha{ {\boldsymbol{\alpha}} }
\graphicspath{{./}{./Figs/}{../Figs/}}

\newtheorem{theorem}{Theorem}[section]
\newtheorem{corollary}{Corollary}[section]
\newtheorem{definition}{Definition}[section]

\newtheorem{lemma}{Lemma}[section]

\begin{document}

\title{Face Recognition using Optimal Representation Ensemble}

\author
{
  Hanxi Li$^{1,2}$, ~  Chunhua Shen$^3$, ~ Yongsheng Gao$^{1,2}$ \\
  $^1$NICTA\thanks{NICTA
  is funded by the Australian Government as represented by the
  Department of Broadband, Communications and the Digital Economy and
  the Australian Research Council through the ICT Center of Excellence
  program.
  }, Queensland Research Laboratory, QLD, Australia\\
  $^2$Griffith University, QLD, Australia  \\
  $^3$University of Adelaide, SA, Australia
}

\maketitle

\begin{abstract}

  Recently, the face recognizers based on linear representations have been shown to
  deliver state-of-the-art performance. These approaches assume that the faces, belonging
  to one individual, reside in a linear-subspace with a respectively low dimensionality.
  In real-world applications, however, face images usually suffer from expressions,
  disguises and random occlusions. The problematic facial parts undermine the validity of
  the subspace assumption and thus the recognition performance deteriorates significantly.
  In this work, we address the problem in a learning-inference-mixed fashion. By observing
  that the linear-subspace assumption is more reliable on certain face patches rather than
  on the holistic face, some \emph{Bayesian Patch Representations} (BPRs) are randomly
  generated and interpreted according to the Bayes' theory. We then train an ensemble
  model over the patch-representations by minimizing the empirical risk \wrt the
  ``leave-one-out margins''.
  The obtained model is termed \emph{Optimal Representation Ensemble} (ORE), since it
  guarantees the optimality from the perspective of Empirical Risk Minimization.
  To handle the unknown patterns in test faces, a robust version of BPR is proposed by
  taking the non-face category into consideration. Equipped with the Robust-BPRs, the
  inference ability of ORE is increased dramatically and several \emph{record-breaking}
  accuracies ($99.9\%$ on Yale-B and $99.5\%$ on AR) and desirable efficiencies (below
  $20$ ms per face in Matlab) are achieved.  It also overwhelms other modular heuristics
  on the faces with random occlusions, extreme expressions and disguises. Furthermore, to
  accommodate immense BPRs sets, a boosting-like algorithm is also derived. The boosted
  model, \aka Boosted-ORE, obtains similar performance to its prototype. Besides the
  empirical superiorities, two desirable features of the proposed methods, namely, the
  \emph{training-determined model-selection} and the \emph{data-weight-free boosting
  procedure}, are also theoretically verified.  They reduce the training complexity
  immensely, while keeps the generalization capacity not changed.

\end{abstract}

\section{Introduction}
\label{sec:introduction}
  
  Face Recognition is a long-standing problem in computer vision.  In the past decade,
  much effort has been devoted to the \emph{Linear Representation} (LR) based algorithms
  such as Nearest Feature Line (NFL) \cite{Li_NN_99_NFL}, Nearest Feature Subspace (NFS)
  \cite{Tzung_PAMI_02_Discriminant}, Sparse Representation Classification (SRC)
  \cite{Wright_PAMI_09_Face} and the most recently proposed Linear Regression
  Classification (LRC) \cite{Nassem_PAMI_10_Face}. Compared with traditional face
  recognition approaches, higher accuracies have been reported.  The underlying assumption
  for the LR-classifiers is that the faces of one individual reside in a low-dimensional
  linear manifold. This assumption, however, is only valid when the cropped faces are
  considered as rigid Lambertian surfaces and without any occlusion
  \cite{Georghiades_PAMI_01_Few, Basri_PAMI_2003_Lambertian}. In practice, the
  linear-subspace model is sometimes too rudimentary to handle expressions, disguises and
  random occlusions which usually occur in local regions, \eg expressions influence the
  mouth and eyes more greatly than the nose, scarves typically have the impact on
  lower-half faces. The problematic face parts are not suitable for performing the linear
  representation and thus reduce the recognition accuracy. On the other hand, there should
  be some face parts which are less problematic, \ie more reliable. But, how can we
  evaluate the reliability of one face part?  Given the reliabilities of all the parts,
  how do we make the final decision? 
  
  Several heuristic methods were introduced to address the problem. In particular, the
  modular approach is used in \cite{Wright_PAMI_09_Face} and \cite{Nassem_PAMI_10_Face}
  for eliminating the adverse impact of continuous occlusions. Significant improvement in
  accuracy was observed from the partition-and-vote \cite{Wright_PAMI_09_Face} or the
  partition-and-compete \cite{Nassem_PAMI_10_Face} strategy. The drawbacks of these
  heuristics are also clear. First, one must roughly know {\em a priori} the shape and
  location of the occlusion otherwise the performance will still deteriorate. It is
  desirable to design more flexible ``models'' to handle occlusions with arbitrary spatial
  features. Furthermore, the existing heuristics discard much useful information, like the
  representation residuals in \cite{Wright_PAMI_09_Face} or the classification results of
  the unselected blocks in \cite{Nassem_PAMI_10_Face}. Higher efficiencies are expected
  when all the information is simultaneously analyzed. Thirdly, there is great potential
  to increase the performance by employing a sophisticated fusion method, rather than the
  primitive rules in \cite{Wright_PAMI_09_Face} and \cite{Nassem_PAMI_10_Face}. Finally,
  most existing methods neglect the fact that the LR-method can also be used to
  distinguish human faces from non-face images, or partly-non-face images. By harnessing
  this power, one could achieve higher robustness to occlusions and noises.
    
  In this paper, we propose a learning-inference-mixed framework to learn and recognize
  faces. The novel framework generate, interpret and aggregate the partial representations
  more elegantly. First of all, LRs are performed on randomly-generated face patches.
  Secondly, in a novel manner, we interpret every patch representation as a probability
  vector, with each element corresponding to a certain individual. The interpretation is
  obtained via applying Bayes theorem on a basic distribution assumption, and thus is
  referred to as Bayesian Patch Representation (BPR). We then learn a linear combination
  of the obtained BPRs to gain much higher classification ability. The combination
  coefficients, \ie the weights associated with different BPRs, are achieved via
  minimizing the exponential loss \wrt sample margins \cite{Herbrich_02_Learning}. In this
  way, most given face-related patterns are learned via assigning different ``importances'' to
  various patches. The learned model is termed \emph{Optimal Representation Ensemble}
  (ORE) since it guarantees the optimality from the perspective of Empirical Risk
  Minimization. To cope with unknown-patterns in test faces, a variation of BPR, namely
  Robust-BPR, is derived by taking account of the \emph{Generic-Face-Confidence}. The
  inference power of the ORE model is improved dramatically by employing the Robust-BPRs. 

  The BPRs are, essentially, instance-based. One can not simply copy the off-the-shelf
  ensemble learning method to combine them. To accommodate the instance-based predictors
  and optimally exploit the given information, we propose the \emph{leave-one-out margin}
  for replacing the conventional margin concept. The leave-one-out margin also makes the
  ORE-Learning procedure extremely resistant to the overfitting, as we theoretically
  verified. One therefore can choose the model parameter merely depending on the training
  errors. This merit of ORE-Learning leads to a remarkable drop in the validation
  complexity.  In addition, to tailor the proposed method to immense BPR sets, a
  boosting-like algorithm is designed to obtain the ORE in an iterative fashion. The
  boosted model, Boosted-ORE, could be learned very efficiently as we prove that the
  training procedure is unrelated to data weights. \emph{From a higher point of view, we
  offer an elegant and efficient framework for training a discriminative ensemble of
  instance-based classifiers}.   
    
  A few work has used ensemble learning methods for face recognition
  \cite{Wang_CVPR_04_Sampling, Chawla_CVPR_05_Subspace, Lu_TNN_05_Ensemble,
  Xiao_CVPR_06_Joint, Wang_IJCV_06_Subsampling}. Nonetheless, those methods only combine
  the model-based, primitive classifiers, \eg Linear Discriminant Analysis (LDA) or
  Principle Component Analysis (PCA), which are sensitive to illumination, easy to
  over-fit and neglect the non-face category. In contrast, our Bayesian-rule-based BPRs
  overcome all these drawbacks. Furthermore, the proposed ensemble methods globally
  minimize an explicit loss function \wrt margins. It serves as a more principled way, in
  comparison to the simple voting strategies \cite{Wang_CVPR_04_Sampling,
  Chawla_CVPR_05_Subspace, Wang_IJCV_06_Subsampling} or the heuristically customized
  boosting schemes \cite{Lu_TNN_05_Ensemble, Xiao_CVPR_06_Joint}. 

  The experiment part justifies the excellence of proposed algorithms over conventional
  LR-methods. In particular, ORE achieves some record-breaking accuracies ($99.9\%$ for
  Yale-B dataset and $99.5\%$ for AR dataset) on the faces with extreme illumination
  changes, expressions and disguises. Boosted-ORE also shows similar recognition
  capability. Equipped with the GFC, Robust-ORE outperforms other modular heuristics under
  all the circumstances. Moreover, the ORE-model also shows the highest efficiency (below
  $20$ ms per face with Matlab and one CPU core) among all the compared LR-methods. 

  The rest of this paper is organized as follows. In Section \ref{sec:back}, we briefly
  introduce the family of LR-classifiers and the modular heuristics. BPR and
  Robust-BPR are proposed in the following section. The learning algorithm for obtaining
  ORE is derived in Section \ref{sec:combine_bpr} where we also prove the validity of the
  training-determined model-selection. The derivations of the boosting-like variation,
  \aka Boosted-ORE, and its desirable feature in terms of ultrafast training are given in
  Section \ref{sec:boost_ore}. Section~\ref{sec:strategy} introduces the
  learning-inference-mixed strategy of the ORE algorithm. The experiment and results are
  shown in Section \ref{sec:exp} while the conclusion and future topics can be found in
  the final section.

\section{Background}
\label{sec:back}

\subsection{The family of LR-classifiers}
\label{subsec:family}
  
  For a face recognition problem, one is usually given $N$ vectorized face images $\X \in
  \mathr^{D \times N}$ belonging to $K$ different individuals, where $D$ is the
  dimensionality of faces and $N$ is the face number. Let us suppose their labels are $\l =
  \{l_1, l_2, \cdots, l_N\},$ $ l_i \in \{1, 2, \cdots, K\}~\forall i$. When a probe
  face $\y \in \mathr^D$ is provided, we need to identify it as one individual exists in
  the training set, \ie $\gamma_{\y} = H(\y) \in \{1, 2, \dots, K\}$, where $H(\cdot)$ is
  the face recognizer that generates the predicted label $\gamma$. Without loss of
  generality, in this paper, we assume all the classes share the same sample number $M =
  N/K$. For the $k$th face category, let $\x^k_i \in \mathr^{D}$ denote the $i$th face
  image and $\X_k  = [\x_1, \x_2, \dots, \x_M] \in \mathr^{D \times M}$ indicates the
  image collection of the $k$th class\footnote{For simplicity, we slightly abuse the
  notation: the symbol of a matrix is also used to represent the set comprised of all the
  columns of this matrix.}. 

  Nearest Neighbor (NN) can be thought of as the most primitive LR-method. It uses only
  one training face, \aka the nearest neighbor, to represent the test face. However,
  without a powerful feature extraction approach, NN usually performs very poorly.
  Therefore, more advanced methods like NFL \cite{Li_NN_99_NFL}, NFS
  \cite{Tzung_PAMI_02_Discriminant}, SRC \cite{Wright_PAMI_09_Face} and LRC
  \cite{Nassem_PAMI_10_Face} are proposed. Most of their formulations (\cite{Li_NN_99_NFL,
  Tzung_PAMI_02_Discriminant, Nassem_PAMI_10_Face}) could be unified. For class $k \in
  \{1, 2, \cdots, K\}$, a typical LR-classifier firstly solve the following problem to get
  the representation coefficients $\Beta^\ast_k$, \ie
  \begin{equation}
      \min_{\Beta_k} ~ \|\y - \tilde{\X}_k\Beta_k\|_2 ~~ \forall ~ k
      \in \{1, 2, \cdots, K\},
    \label{equ:opt_lr}
  \end{equation}
  where $\|\cdot\|_p$ stands for the $\ell_p$ norm and $\tilde{\X}_k$ is a subset of
  $\X_k$, selected under certain rules. The above problem, also known as the \emph{Least
  Square Problem}, has a closed-form solution given by
  \begin{equation}
    \boldsymbol{\Beta}^{\ast}_k = (\tilde{\X}_k^\T \tilde{\X}_k)^{-1}{\tilde{\X}_k}^\T
    \y.
    \label{equ:sol_ls}
  \end{equation}
  The identity of test face $\y$ is then retrieved as 
  \begin{equation}
    \gamma_\y = \argmin_{k \in \{1, \cdots, K\}} r_k,
    \label{equ:label_lr}
  \end{equation}
  where $r_k$ is the reconstruction residual associated with class $k$, \ie
  \begin{equation}
    r_k = \|\y - \tilde{\X}_k\Beta^\ast_k\|_2.
    \label{equ:residual}
  \end{equation}
  
  Different rules for selecting $\tilde{\X}_k$ actually specify different members of the
  LR-family. NN merely use one nearest neighbor from $\X_k$ as the representation basis;
  NFL exhaustively searches two faces which form a nearest line to the test face; NFS
  conduct a similar search for the nearest subspace with a specific dimensionality;
  Finally, at the other end of the spectrum, LRC directly employ the whole $\X_k$ to
  represent $\y$. Note that although the solution of problem \eqref{equ:opt_lr} is
  closed-form, most LR-method requires a brute-force search to obtain $\tilde{\X}_k$. The
  only exception exists in LRC where $\tilde{\X}_k = \X_k$, thus LRC is much more faster
  than the other members. 
  
  The SRC algorithm, on the other hand, solves a second-order-cone problem over the entire
  training set $\X$. The optimization problem writes:
  \begin{equation}
      \min_{\Beta} ~ ~ \|\Beta\|_1 ~ ~ ~ \sst ~ ~ \|\y - \X\Beta\|_2 \le \varepsilon.
    \label{equ:opt_src}
  \end{equation}
  Then the representation coefficients for class $k$ are calculated as:
  \begin{equation}
      \Beta^\ast_k = \delta_k(\Beta) ~~ \forall ~ k \in \{1, 2, \cdots, K\},
    \label{equ:rep_para_src}
  \end{equation}
  where function $\delta_k(\Beta)$ sets all the coefficients of $\Beta$ to $0$ except
  those corresponding to the $k$th class \cite{Wright_PAMI_09_Face}. The identifying
  procedure of SRC is the same\footnote{Note that for SRC, $\tilde{\X}_k = \X_k$.} to
  \eqref{equ:label_lr}. By treating the occlusion as a ``noisy'' part, Wright \etal
  \cite{Wright_PAMI_09_Face} also proposed a robust version of SRC, which conducts the
  optimization as follows:
  \begin{equation}
      \min_{\bomega} ~ ~ \|\bomega\|_1 ~ ~ ~ \sst ~ ~ \|\y - [\X,~ \I]\bomega\|_2 \le
      \varepsilon,
    \label{equ:opt_robust_src}
  \end{equation}
  where $\I$ is an identity matrix, $\bomega = [\Beta,~ \e]^\T$ and $\e$ is the
  representation coefficients corresponding to the non-face part. SRC is very slow
  \cite{Shi_CVPR_10_Face} due to the second-order-cone programming and the assumption is
  also doubtful \cite{Shi_CVPR_11_Face}. 

  Obviously, all the LR-methods are \emph{generative} rather than \emph{discriminative}.
  Their main goal is to best reconstruct test face $\y$, while the subsequent
  classification procedure seems a ``byproduct''. Nonetheless, they still achieved
  impressive performances because the underlying linear-subspace theory
  \cite{Georghiades_PAMI_01_Few, Basri_PAMI_2003_Lambertian} keeps approximately valid no
  matter how the illumination changes. Unfortunately, for the face with extreme
  expressions, disguises or random contaminations, the theory doesn't hold anymore and
  poor recognition accuracies are usually observed.

\subsection{Two modular heuristics for robust recognition}
\label{subsec:modular}

  To ease the difficulties, some modular methods are proposed. In particular, Wright \etal
  \cite{Wright_PAMI_09_Face} partition the face image into several (usually $4$ to $8$)
  blocks and perform the robust SRC, as illustrated in \eqref{equ:opt_robust_src}, on each
  of them. The final identity of the test face is determined via a majority voting over
  all the blocks. We term this algorithm as Block-SRC in this paper. The \emph{Distance
  based Evidence Fusion} (DEF) \etal \cite{Nassem_PAMI_10_Face} modifies the LRC via a
  similar block-wise strategy while the predict label is given by a competition procedure.
  Without loss of generality, their methods can be summarized as: 
  \begin{equation}
    \gamma = F\left(\r^1, \r^2, \dots, \r^T\right),
    \label{equ:fusion}
  \end{equation}
  where $\r_t = [r_{t,1}, r_{t, 2}, \cdots, r_{t, K}]^{\T}$ is the collection of all the $K$
  residuals for the $t$th block, $F(\cdot)$ refers to the fusion method which counts the
  votes in \cite{Wright_PAMI_09_Face} and perform the competition in
  \cite{Nassem_PAMI_10_Face}. In other words, $F(\cdot)$ reflects how we combine the
  block-representations' outputs.

  The modular approaches did increase the accuracy. Their success implies that the
  linear-subspace assumption is more reliable on certain face parts, rather than the
  holistic face. Nonetheless, their drawbacks, as described in the introduction part, are
  also obvious. In the following sections, we build a more elegant framework to generate,
  interpret and aggregate the partial representations.

\section{Bayesian Patch Representation}
\label{sec:bpr}

\subsection{Random face patches}
\label{subsec:random_face_patch}

\subsubsection{What are random face patches}
\label{subsubsec:what}

  A random face patch is a continuous part of the face image, with an arbitrary shape and
  size. The method based on image patches has illustrated a great success in face
  detection \cite{Viola2004}. Differing from the Haar-feature in face detection, linear
  representations are much more sophisticated. There is no need to generate the patches
  exhaustively. In this paper, we only employ $500$ small patches randomly distributing
  over the face image. Those patches are already sufficient to sample all the reliable
  face parts. Different weights are assigned to these patches to indicate their
  importances for a specific recognition task. We expect that a certain combination of
  these patches could yield similar classification capacity to the direct use of all the
  reliable regions. 

  Figure~\ref{subfig:weighted_patches} gives us an example of the weighted patches. $500$
  random face patches are generated with different shapes (here only rectangles). The
  higher its weight is assigned, the redder and wider a patch is shown. The weights are
  obtained by using the proposed ORE-Learning algorithm on AR \cite{Martinez_98_AR}
  dataset. Note that most patches are purely blue which implies their weights are too
  small to influence the classification. We simply ignore those patches in practices.

  
\subsubsection{Why random face patches}
\label{subsubsec:what}

  \begin{figure}[ht!]
    \centering
    \subfigure[weighted patches]{\label{subfig:weighted_patches}\includegraphics[width=0.25\textwidth]{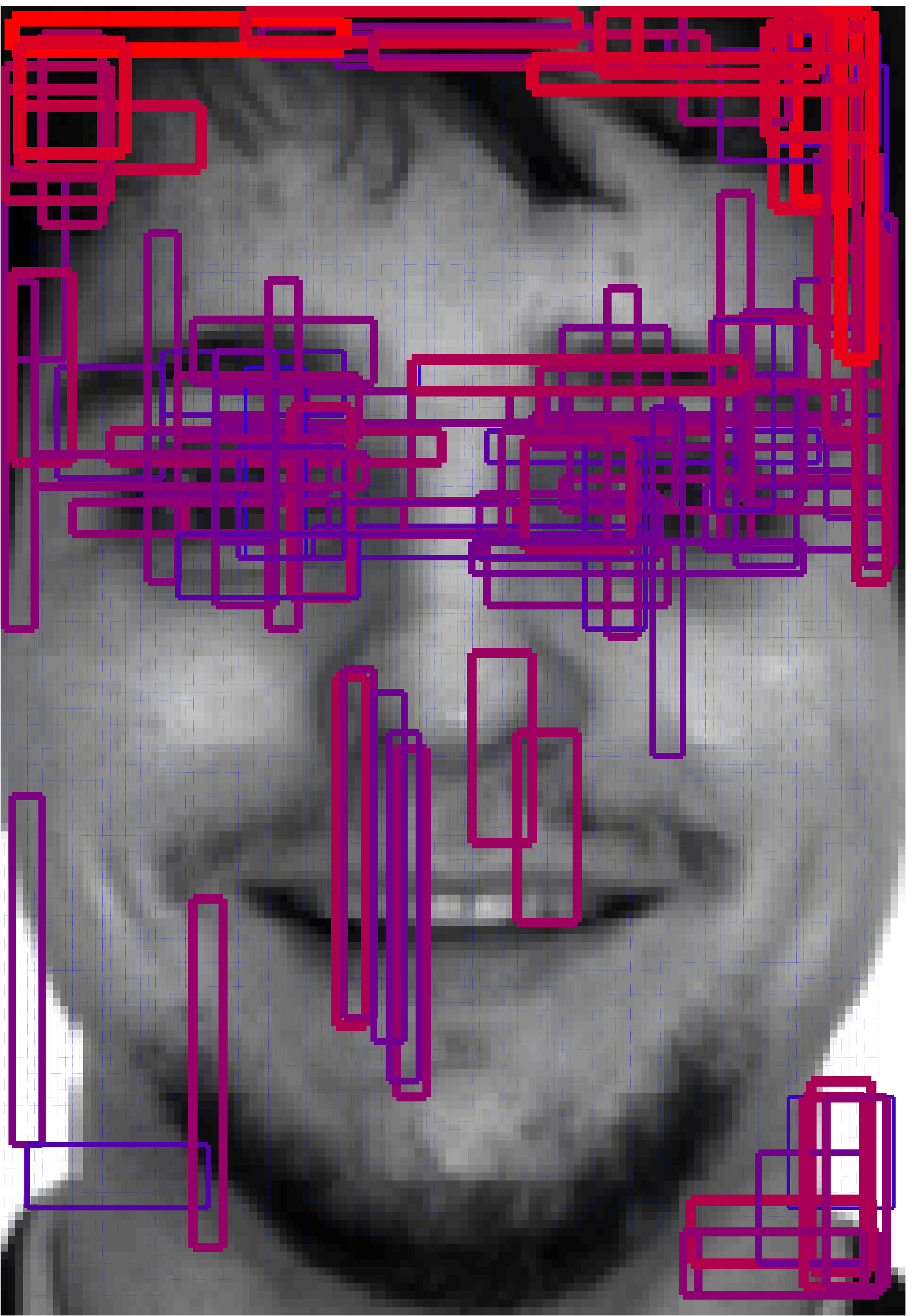}}
    \subfigure[pixel-energy map]{\label{subfig:pixel_weight}\includegraphics[width=0.25\textwidth]{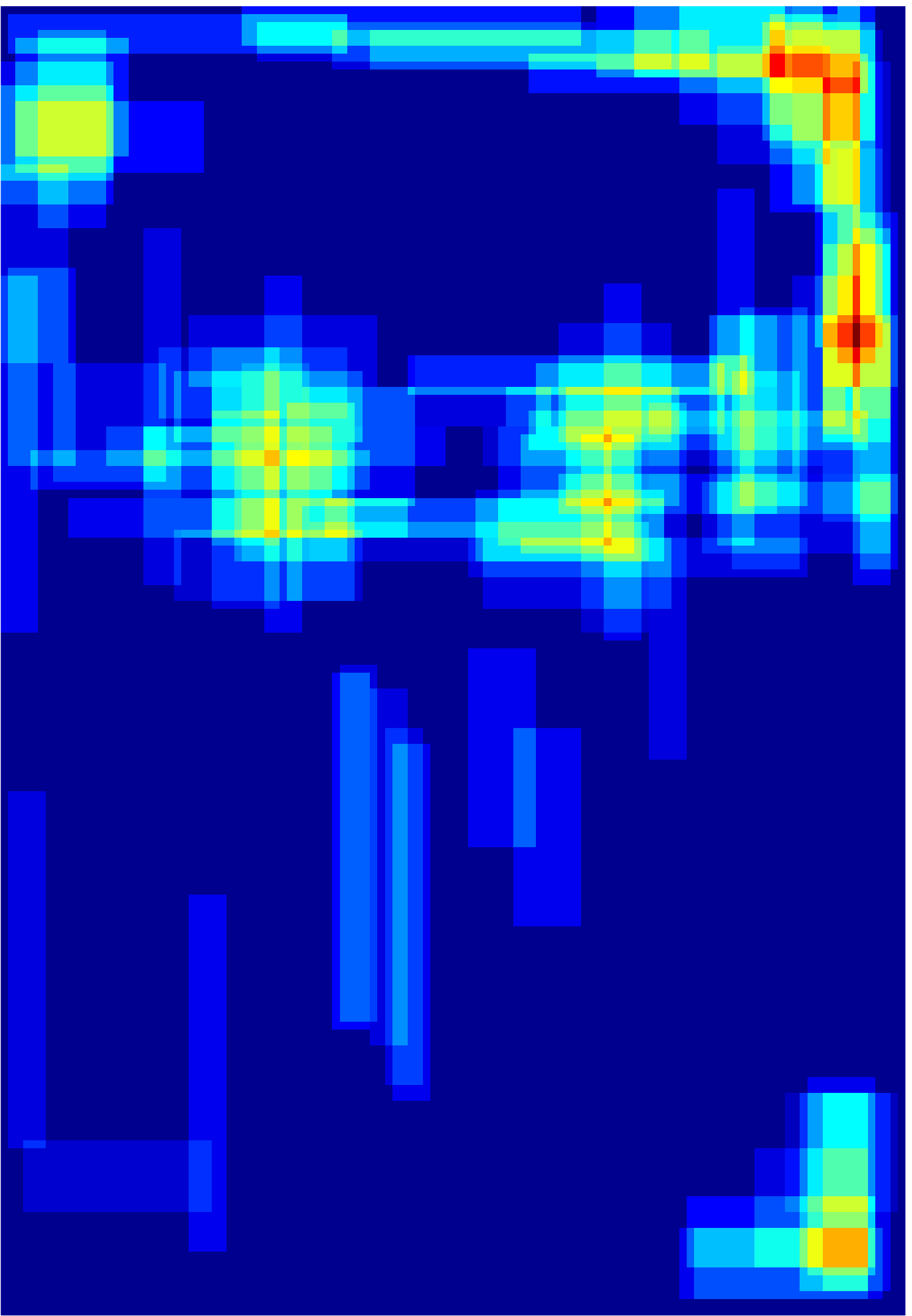}}
    \subfigure[focused face]{\label{subfig:focused_face}\includegraphics[width=0.25\textwidth]{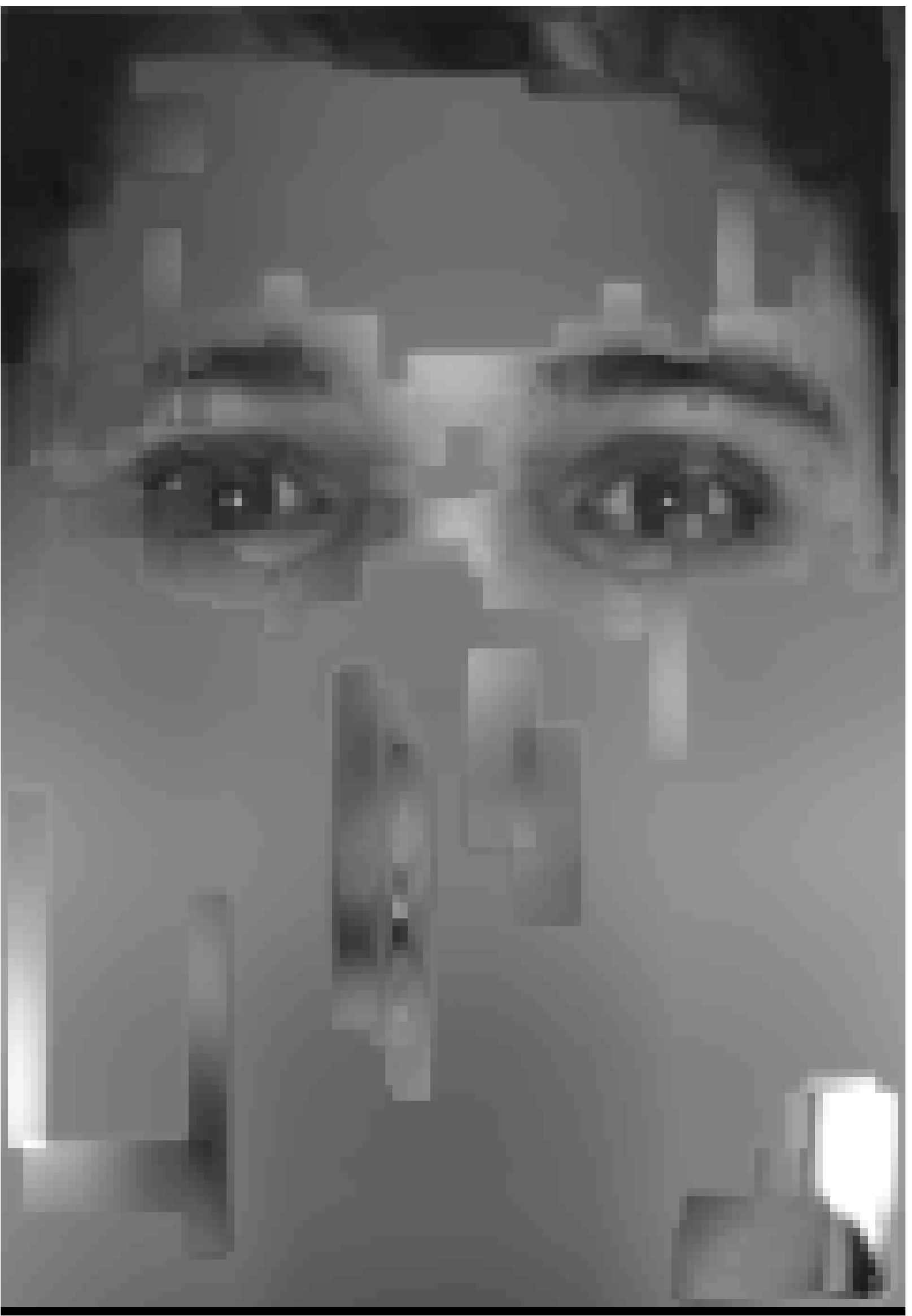}}
    \caption
    {
        The demonstration of random face patches. (a): $500$ random face patches with
        different weights. The weight is represented by the color and width of edges.
        (b): the corresponding pixel-energy map, the energy of one pixel is defined as the
        average weight of all the overlapping patches. (c): the simulated focusing
        behavior. Only a small part of the face is emphasized (focused) while the
        others are ignored (blurred). The weights are obtained by using the proposed
        ORE-Learning algorithm on AR \cite{Martinez_98_AR} dataset.
    }
   \label{fig:patch_demo}
  \end{figure}

  Compared with the deterministic blocks, random patches have the following two
  advantages. 
  
  \begin{itemize}
    \item 

      More flexible. The reliable region of a given face could be in arbitrary shape and
      location. Deterministic blocks are therefore too rudimentary to represent it. The
      random patch approach, on the other hand, can approximate any shape of interested
      regions. Figure~\ref{subfig:pixel_weight} illustrates a pixel-energy map
      corresponding to the patches shown in Figure~\ref{subfig:weighted_patches}. A
      pixel's energy is the mean weights of all the patches covering this pixel. A
      irregularly-shaped but reasonable region, which includes two eyes and certain parts
      of the forehead, is emphasized.  From a bionic perspective, it is promising to
      aggregate the random patches for simulating the focusing behavior of human beings.
      Figure~\ref{subfig:focused_face} illustrates the simulated focusing behavior: the
      face is blurred according to the pixel-weights, only the focused facial part, \aka
      the emphasized region, remains clear.


    \item 

      More efficient. According to Figure~\ref{fig:patch_demo}, only a limited number
      (always in the order of $10^1$ in this work) of patches are taken into
      consideration. As we empirically proved in the experiment, the complexity of
      performing the LR method on several small patches are usually lower than that for
      few large blocks.

  \end{itemize}

\subsection{Bayesian Patch Representation}
\label{subsec:bpr}

  Given that the linear-subspace assumption is more reliable on certain face patches, it
  is intuitive to perform the LR-method for each patch. In principle, we could employ
  either member of the LR-family to perform the linear representation on the patches.
  According to the theoretical analysis \cite{Georghiades_PAMI_01_Few,
  Basri_PAMI_2003_Lambertian}, however, it seems no need to specifically select a certain
  subset from $\X_k$. We thus employ the whole $\X_k$ to form the representation basis,
  just like what \cite{Nassem_PAMI_10_Face, Georghiades_PAMI_01_Few} did. 
  In particular, for class $k$ and patch $t$, we denote the patch set as $\X^t_k =
  [\x^t_1, \x^t_2, \cdots, \x^t_M] \in \mathr^{d \times M}$, with each column obtained via
  vectorizing a image patch. The representation coefficients $\Beta^{\ast}_{t,k}$, for the
  $k$th class and $t$th patch is then given by 
  \begin{equation}
    \Beta^{\ast}_{t,k} = ({\X^t_k}^\T \X^t_k)^{-1}{\X^t_k}^\T \y.
    \label{equ:sol_patch_lr}
  \end{equation}
  Then the residual $r_{t,k}$ can be obtained as
  \begin{equation}
    r_{t,k} = \|\y_t - \X^t_k\Beta^{\ast}_{t,k}\|_2,
    \label{equ:residual_patch}
  \end{equation}
  where $\y_t$ is the cropped test image according to the patch location. In this paper, all
  the patches are normalized so that their $\ell_2$ norms are equal to $1$. As a result,
  $r_{t,k} \in [0,1],~\forall t,k$.
  
  Ordinary LR-methods, including the robust variations, only focus on the smallest
  residual or the corresponding class label. This strategy will lose much useful
  information. Differing from the conventional manner, we interpret every patch
  representation as a probability vector $\b_t$. The $k$th element of $\b_t$, namely
  $b_{t,k}$, is the probability that current test patch $\y_t$ belongs to individual $k$
  \ie
  \begin{equation}
    b_{t,k} = \Prob{\gamma_\y = k \mid \y_t}.
    \label{equ:post_bt}
  \end{equation}

  We obtain the above posteriors by applying the Bayesian theorem. First of all, it is
  common that all the classes share the same prior probability, \ie $\Prob{\gamma_\y = k} =
  1/K,~\forall~ k$. The linear-subspace assumption states that, if one test face belongs
  to class $k$, the test patch $\y_t$ should distribute around the linear-subspace spanned
  by $\X^t_k$. The probability of a remote $\y_t$ is smaller than the one close to the
  subspace. In this sense, when the category is known, we can assume the random variable
  $\y_t$ belongs to a distribution with the probability density function 
  \begin{equation}
    \Prob{\y_t \mid \gamma_\y = k} = C \cdot \exp(-{r^2_{t,k}}/{\delta}), 
    \label{equ:condition_prob}
  \end{equation} 
  where $\delta$ is a assumed variance and the $C$ is the normalization factor. This
  distribution, in essence, is a \emph{singular normal distribution} as its covariance
  matrix is singular. Figure~\ref{fig:singular_normal_distribution} depicts the tailored
  distribution in a $2$-D space for the linear-subspace assumption. 
  
  \begin{figure}[h]
    \centering
    \includegraphics[width=0.58\textwidth]{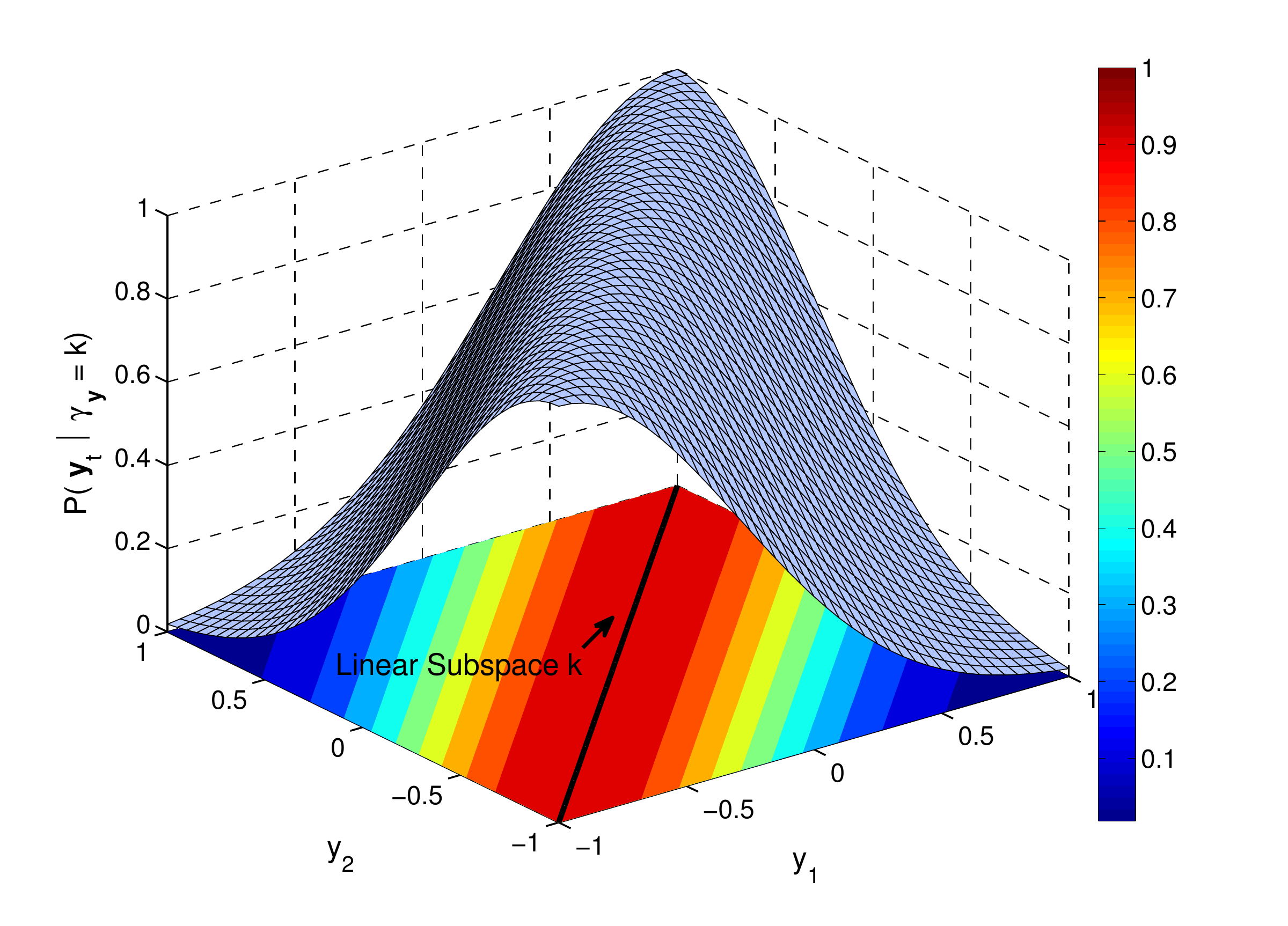}
    \caption
    {
       The demonstration of the singular normal distribution tailored to the
       linear-subspace assumption. The black line indicates the linear-subspace $k$ while
       different colors represent different probabilities. The surface of the probability
       density function is also shown above. Note that here $D = 2$ thus the subspace can
       have the dimensionality at most $1$, or in other words, a line. 
    }
   \label{fig:singular_normal_distribution}
  \end{figure}

  According to Bayes' rule, the posterior probability is then derived as 
  \begin{equation}
    \begin{split}
      b_{t,k} & = \frac{\Prob{\y_t \mid \gamma_\y =
      k}\cdot\Prob{\gamma_\y = k}}{\sum_{j = 1}^{K}\Prob{\y_t \mid \gamma_\y =
      j}\cdot\Prob{\gamma_\y = j}} \\
      & = \frac{C/K \cdot \exp(-r^2_{t,k}/\delta)}{\sum_{j=1}^{K}C/K\cdot\exp(-r^2_{t, j}/\delta)} \\
      & = \frac{\exp(-r^2_{t,k}/\delta)}{\sum_{j=1}^{K}\exp(-r^2_{t, j}/\delta)}.
    \end{split}
    \label{equ:derive_post}
  \end{equation}
  As an example, Figure~\ref{fig:post_distribution} shows the distribution of the
  posterior $b_{t,1}$ when there are only $2$ orthogonal linear-subspaces ($1$ and $2$)
  and dimensionality $D = 2$

  \begin{figure}[h]
    \centering
    \includegraphics[width=0.58\textwidth]{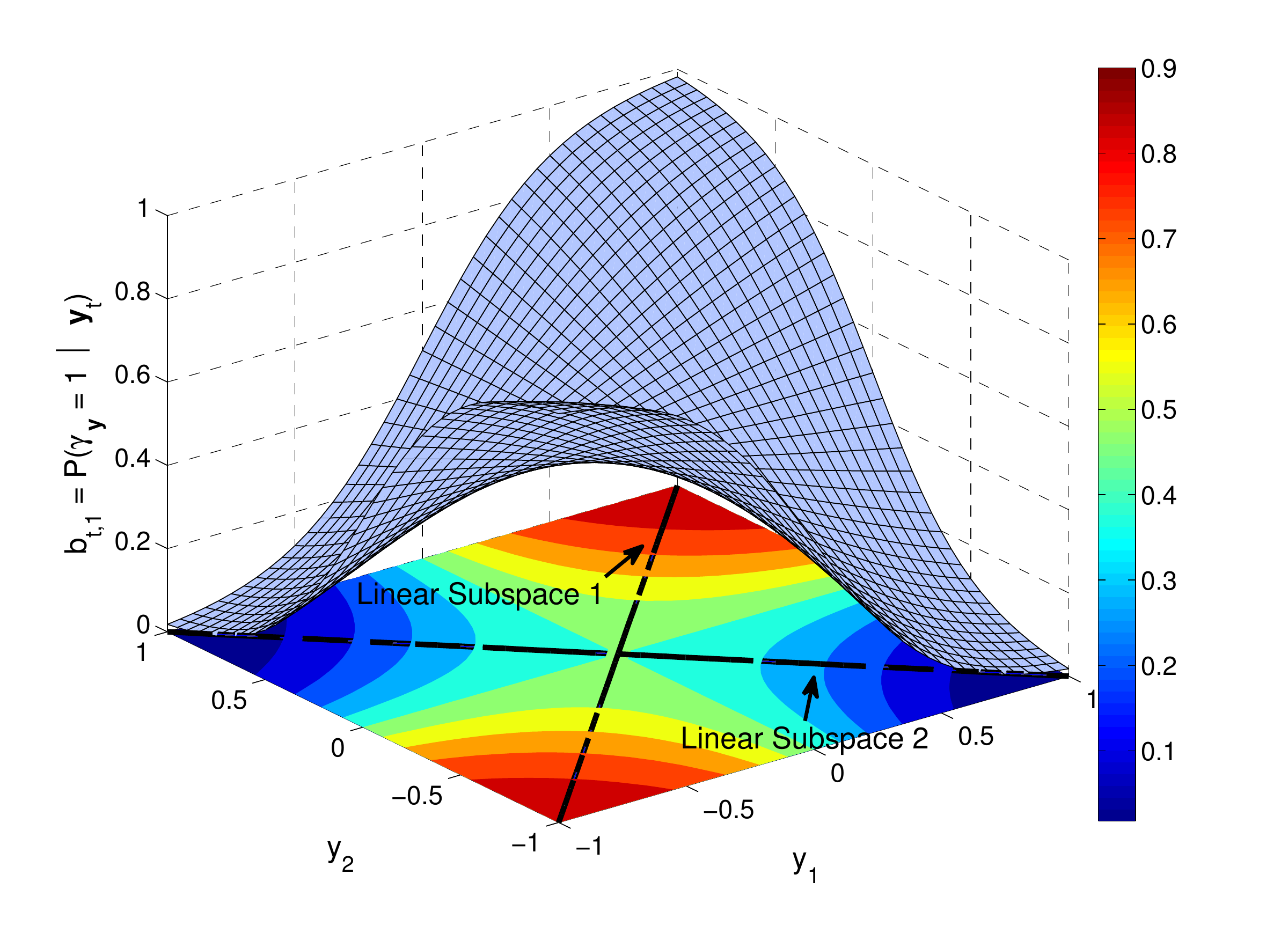}
    \caption
    {
      The demonstration of the posterior distribution $b_{t,1} = \Prob{\gamma_\y = 1 \mid
      \y_t}$. The dimensionality of the original space is $2$. In this example, two
      linear-subspace, \ie the two black lines, are orthogonal to each other. 
    }
   \label{fig:post_distribution}
  \end{figure}
  
  We finally aggregate all the posteriors into a vector $\b_t = [b_{t,1},$ $b_{t,2},$
  $\cdots,$ $b_{t,K}]^{\T}$. The elegant interpretation $\b_t$, termed Bayesian Patch
  Representation (BPR), keeps most information related to the representation and thus
  could lead to a more accurate recognition result. In practice, it makes little sense to
  impose a constant $\delta$ for all the patches and faces. We thus use normalized one,
  \begin{equation}
    \delta_t = 0.1\cdot\min_{k}(r^2_{t,k}),
    \label{equ:prac_var}
  \end{equation}
  for the $t$th patch.
 
\section{Combine the Bayesian Patch Representations}
\label{sec:combine_bpr}

\subsection{Learn a BPR ensemble via Empirical Risk Minimization}
\label{subsec:erm}

  Besides the interpretation, the aggregation method is also vital for the final
  classification. The existing fusion rules, as shown in \eqref{equ:fusion}, are
  rudimentary and non-parameterized thus hard to optimize. In the machine learning
  community, classifier-ensembles learned via an Empirical Risk Minimization process are
  considered to be more powerful than the simple methods \cite{Schapire_ML_99_Improved,
  Friedman_AS_00_Special}. 
  
  As a consequence, we linearly combine the BPRs to generate a predicting vector
  $\bxi(\y)$$ = $$[\xi_1(\y), \xi_2(\y),$$ \dots, $$\xi_K(\y)]^{\T}$$\in \mathr^{K}$, \ie
  \begin{equation}
    \bxi(\y) = \sum_{t =
    1}^{T}\alpha_t\b_t(\y) = \B(\y)\Alpha,
    \label{equ:ensemble_model}
  \end{equation}
  with $\xi_k(\y)$ indicating the confidence that $\y$ belongs to the $k$th class, and
  $\Alpha = [\alpha_1, \alpha_2, \dots, \alpha_T]^{\T} \psd \mathbf{0}$. The identity of
  test face $\y$ is then given by
  \begin{equation}
    \gamma_\y = \argmin_{k \in \{1, \cdots, K\}} \xi_k(\y)
    \label{equ:ensemble_id}
  \end{equation}

  This kind of linear model dominates the supervised learning literature as it is flexible
  and feasible to learn. The parameter vector $\Alpha$ is optimized via minimizing the
  following Empirical Risk
  \begin{equation}
    \ER = \sum_i^N\loss(z_i) + \lambda\cdot\reg(\Alpha),
    \label{equ:erm}
  \end{equation}
  where $\loss(\cdot)$ is a certain loss function, $\reg(\cdot)$ is the regularization
  term and $\lambda$ is the trade-off parameter. The margin $z_i =
  \mathcal{Z}(l_i,~\bxi(\x_i))$ reflects the confidence that $\bxi$ select the correct
  label for $\x_i$. Specifically, for binary classifications, 
  \begin{equation}
    z_i = \xi_{l_i}(\x_i) - \xi_{l^{'}}(\x_i),~ ~ l^{'} \ne l_i.
    \label{equ:margin}
  \end{equation}
  For multiple-class problems, however, there is no perfect formulation of $z_i$. We then
  intuitively define the $z_i$ as 
  \begin{equation}
    z_i = \frac{1}{K - 1}\sum_{j \ne l_i}^{K}\left(\xi_{l_i}(\x_i) -
    \xi_{j}(\x_i)\right),
    \label{equ:margin_multiple}
  \end{equation}
  \ie the mean of all the ``bi-class margins''. Recall that $\sum_{j = 1}^K b_{t,j}(\x_i)
  = 1$, we then arrive at a simpler definition of $z_i$, 
  \begin{equation}
    z_i = \frac{K}{K-1}\sum_{t = 1}^T \alpha_t \left(b_{t,{l_i}}(\x_i) -
    \frac{1}{K}\right).
    \label{equ:margin_simple}
  \end{equation}
  By absorbing the constant $K/(K - 1)$ into each $\alpha_t$, we have 
  \begin{equation}
    z_i = \sum_{t = 1}^T \alpha_t \left(b_{t,{l_i}}(\x_i) - \frac{1}{K}\right).
    \label{equ:margin_simple}
  \end{equation}
  The term $b_{t,{l_i}}(\x_i) - 1/K$ can be though of as the confidence gap between using
  the $t$th BPR and using a \emph{random guess}. The larger the gap, the more powerful
  this BPR is. Consequently, $z_i$ is the weighted sum of all the gaps, which measures the
  predicting capability of $\bxi(\x_i)$.

  The selection for the loss function and the regularization function has been extensively
  studied in the machine learning literature \cite{Hastie_MI_2005_Elements}.  Among all
  the convex loss formulations, we choose the exponential loss $\loss(z_i) = \exp(-z_i)$,
  motivated by its success in combining weak classifiers\cite{Friedman_AS_00_Special,
  Shen_PAMI_10_Dual}. The $\ell_1$ norm is adopted as our regularization method since it
  encourages the sparsity of $\Alpha$, which is desirable when we want an efficient
  ensemble. Finally, the optimization problem in this paper is given by:
  \begin{equation}
    \begin{split}
    \min_{\Alpha}~&\sum_i^N\exp\left(-\sum_{t = 1}^T \alpha_t
    \left(b_{t,{l_i}}(\x_i) - \frac{1}{K}\right)\right) \\
      \sst ~& \Alpha \psd \mathbf{0}, ~ ~\|\Alpha\|_1 \le \lambda
    \end{split}
    \label{equ:opt}
  \end{equation}
  Note that for easing the optimization, we convert the regularization term to a
  constraint. With an appropriate $\lambda$, this conversion won't change the optimization
  result \cite{Boyd_04_Convex}. The optimization problem is convex and can be solved by
  using one of the off-the-shelf optimization tools such as Mosek \cite{Mosek} or CVX
  \cite{Grant_08_CVX}. The learned model is termed \emph{Optimal Representation Ensemble}
  (ORE) as it guarantees the global optimality of $\Alpha$ from the perspective of
  Empirical Risk Minimization. The learning algorithm for achieving the ORE is referred to
  as ORE-Learning.

\subsection{Leave-one-out margin}
\label{subsec:loo-margin}

  It would be simple to calculate the margin $z_i$ if the BPR were model-based, \ie
  $\b_t(\cdot)$ was a set of explicit functions. In fact, that is the situation for most
  ensemble learning approaches. Unfortunately, that is not the case in this paper where
  $\b_t(\cdot)$ is actually instance-based. 
  
  For a BPR, we always need a \emph{gallery}, \aka the representation basis, to calculate
  $\b_t(\cdot)$. And ideally, the gallery should be the same for both training and test,
  otherwise the learned model is only optimal for the training gallery. Nonetheless, we
  can not directly use the training set, which is the test gallery, as the training
  gallery. Any training sample $\x_i$ will be perfectly represented by the whole training
  set because $\x_i$ itself is in the basis. Consequently, all BPRs will generate
  identical outputs and the learned weights $\alpha_t,~\forall t$ will also be the same.
  To further divide the training set into one basis and one validation set, of course, is
  a feasible solution. However, it will reduce the classification power of ORE as the
  larger basis usually implies higher accuracies. 
  
  To get around this problem, we employ a leave-one-out strategy to utilize as many
  training instances as possible for representations. For every training sample $\x_i$,
  its gallery is given by 
  \[\x_i^C = \X \backslash \x_i,\] 
  \ie the complement of $\x_i$ \wrt the universe $\X$. The leave-$\x_i$-out BPRs, referred
  to as $\b^{\x^C_i}_t(\x_i),~\forall t$, are yielded based on the gallery $\x_i^C$. The
  leave-one-out margin $z_i$ is then calculated as
  \begin{equation}
    z_i = \sum_{t = 1}^T \alpha_t \left(b^{\x_i^C}_{t,{l_i}}(\x_i) - \frac{1}{K}\right).
    \label{equ:loo_margin}
  \end{equation}
  
  In this way, the size of the training gallery is always $N - 1$, we can approximately
  consider the learned $\Alpha^{\ast}$ as optimal for the test gallery $\X$ with the size of $N$.

  After $\Alpha^{\ast}$ is obtained, we also calculate the leave-one-out predicting
  vector as 
  \begin{equation}
    \bxi^{\x^C_i}(\x_i) = \B^{\x^C_i}(\x_i)\Alpha^{\ast},
    \label{equ:loo_predict}
  \end{equation}
  where $\B^{\x^C_i}(\x_i)$ is the collection of the leave-one-out BPRs. The training error
  of the ORE-Learning is given by
  \begin{equation}
    e_{trn} = \frac{1}{N}\sum_{i = 1}^{N}\llbracket\argmax_{k}\xi^{\x^C_i}_k(\x_i) \ne
    l_i\rrbracket,
    \label{equ:train_error}
  \end{equation}
  where $\llbracket\cdot\rrbracket$ denote the boolean operator.  This training error, as
  illustrated below, plays a crucial role in the model-selection procedure of
  ORE-Learning.
    
\subsection{Training-determined model-selection}
\label{subsec:training-determined}
  
  Another issue arising here is how to select a proper parameter $\lambda$ for the
  ORE-Learning. Usually, a validation method such as the \emph{$n$-fold cross-validation}
  is performed to select the optimal parameter among candidates. The validation method,
  however, is expensive in terms of computation, because one needs to repeat the extra
  ``subset training'' for $n$ times and usually $n \ge 5$. 
  From the instance-based perspective, a cross-validation is also unacceptable. In every
  ``fold'' of a $n$-fold cross-validation, we only use a part of training samples as the
  gallery. The setting contradicts the principle that one needs to keep the representation
  basis similar over all the stages. 

  Fortunately, the leave-one-out margin provides the ORE-Learning an advantage: The
  training error of the ORE-Learning serves as a good estimate to its leave-one-out error.
  \emph{We can directly use the training error to select the model-parameter $\lambda$.}
  To understand this, let's firstly recall the definition of the leave-one-out error.  

  \begin{definition}({\bf Leave-one-out error} \cite{Herbrich_02_Learning}) 
    Suppose that $\mathX^N$ denotes a training set space comprised of the training
    sets with $N$ samples $\{\x_1, \x_2, \cdots, \x_N\}$. Given an algorithm $\mathcal{A}:
    \bigcup_{N = 1}^{\infty}\mathX^N \rightarrow \mathcal{F}$, where $\mathcal{F}$ is
    the functional space of classifiers. The leave-one-out error is defined by 
    \begin{equation}
      e_{loo} \triangleq \frac{1}{N}\sum_{i = 1}^{N}\llbracket
      F^{\mathcal{A}}_{\x_i^C}(\x_i) \ne l_i \rrbracket, \label{equ:loo_err}
    \end{equation}
  \end{definition}
  where $F^{\mathcal{A}}_{\x_i^C} = \mathcal{A}(\mathX^N \backslash \x_i)$, \ie the
  classifier learned using $\mathcal{A}$ based on the set $\mathX^N \backslash \x_i$.
  
  The leave-one-out error is known as an unbiased estimate for the generalization error
  \cite{Herbrich_02_Learning, Evgeniou_ML_04_LOO}. Our target in this section is to build
  the connection between $e_{loo}$ and $e_{trn}$ for ORE-Learning. Suppose that all
  the training faces are non-disguised, which is the common situation, then let us make
  the following basic assumption. 
  
  {\bf Assumption:} One patch-location $t$ on the human face could be affected by $Q_t$
  different expressions. Every expression leads to a distinct and convex Lambertian
  surface.
  
  According the theory in \cite{Georghiades_PAMI_01_Few} and
  \cite{Basri_PAMI_2003_Lambertian}, the different appearances of one patch surface,
  caused by illumination changes, span a linear-subspace with a small dimensionality
  $\Phi$. Given that $M$ training patches from the patch-location $t$ is collected in
  $\X^t_k$, its arbitrary subset $\X^t_P$ contains $P$ ($P < \Phi \ll M$) samples. With
  the assumption, we can verify the following lemma.
  \begin{lemma}({\bf The stability of BPRs})
    If the training subset $\X^t_k$ contains at least $(\Phi Q_t + P)$ \iid patch
    samples, set $\X^t_k$ and set $\X^t_k \backslash \X^t_P$ share the same representation
    basis. 
  \label{lemma:1}
  \end{lemma}
  \begin{proof}
    Let us denote the linear-subspace formed by $\X^t_k$ as $\mathcal{U}^t$. $\mathcal{U}^t_q$
    refers to its subset spanned by the patches associated with surface $q$. We know that 
    \begin{equation*} 
      \rank(\mathcal{U}^t) \le \sum_{q = 1}^{Q_t}\rank(\mathcal{U}^t_q) = \Phi Q_t.  
    \end{equation*}
    When $\X^t_P$ is moved out, the new linear-subspace spanned by $\X^t_k \backslash
    \X^t_P$ is denoted by $\tilde{\mathcal{U}}^t$. According to the given condition, there
    are still $\Phi Q_t$ \iid patches remaining. Then with an overwhelming probability,
    \begin{equation*}
      \rank(\tilde{\mathcal{U}}^t) = \Phi Q_t = \rank(\mathcal{U}^t).  
    \end{equation*}
    Considering that $\tilde{\mathcal{U}}^t \subset \mathcal{U}^t$, we then arrive at 
    \begin{equation*}
     \tilde{\mathcal{U}}^t = \mathcal{U}^t.
    \end{equation*}
    The space remains the same, so does its basis.
  \end{proof}

  In contrast, other classifiers, such as decision trees or linear-LDA-classifiers, don't
  have this desirable stability. They always depend on the exact data, rather than the
  extracted space-basis.
 
  Note that $\Phi$ is usually very small \cite{Georghiades_PAMI_01_Few,
  Basri_PAMI_2003_Lambertian}. The value of $Q_t$ is determined by the types of
  expressions that can affect patch $t$. It is also very limited if we only consider the
  common ones. That is to say, \emph{with a reasonable number of training samples, the
  BPRs is stable \wrt the data fluctuation.} Specifically, when $\x_i$ is left out ($P =
  1$), all the BPRs' values on samples $\{\x_j \in \X \mid j \ne i\}$ won't change, \ie
  \begin{equation}
    \b^{\x^C_j}_t(\x_j) = \b^{\x^C_{i,j}}_t(\x_j), ~\forall i \ne j,~i,j \in \{1, 2,
    \cdots, N\},
    \label{equ:same_basis}
  \end{equation}
  where $\x^C_{i,j}$ stands for the complement of set $\{\x_i, \x_j\}$. 
  From the perspective of ensemble learning, the original ORE-Learning problem
  $\mathcal{A}(\mathX^N)$ and the leave-$\x_i$-out problem $\mathcal{A}(\mathX^N
  \backslash \x_i)$ share the same ``basic hypotheses''
  \[\b_t(\x), ~\forall t \in \{1, 2, \cdots, T\},\]
  and constraints 
  \[\Alpha \psd \mathbf{0}~\&~\|\Alpha\|_1 \le \lambda.\]
  The only difference is that the former problem involves one more training sample,
  $\x_i$. We know that usually $N \gg 1$, thus one can approximately consider their
  solutions are the same, \ie
  \begin{equation}
    \Alpha^{\ast}_{\x_i^C} = \Alpha^{\ast}, ~ ~\forall i,
    \label{equ:similar_alpha}
  \end{equation}
  where $\Alpha^{\ast}_{\x_i^C}$ is the optimal solution for problem
  $\mathcal{A}(\mathX^N \backslash \x_i)$. Finally, we arrive at the following theorem 
  \begin{theorem}
    With Equation~\eqref{equ:similar_alpha} holding, the training error of the
    ORE-Learning exactly equals to its leave-one-out error. 
    \label{theo:loo}
  \end{theorem}
  \begin{proof}
    In the context of ORE, all types of errors are determined by the predicting vectors
    $\bxi(\x_i),~\forall i$. For the leave-one-out error, we know that 
    \begin{equation}
      \bxi^{loo}(\x_i) = \B^{\x^C_i}(\x_i)\Alpha^{\ast}_{\x_i^C} ~ ~\forall i,
      \label{predict_loo_err}
    \end{equation}
    where $\B^{\x^C_i}(\x_i)$ is defined in \eqref{equ:loo_predict}. Recall that
    \begin{equation*}
      \bxi^{trn}(\x_i) = \bxi^{\x^C_i}(\x_i) = \B^{\x^C_i}(\x_i)\Alpha^{\ast}, ~
      ~ \forall i.
    \end{equation*}
    If Equation~\eqref{equ:similar_alpha} is valid, then obviously, $\bxi^{trn}(\x_i) =
    \bxi^{loo}(\x_i)$. Finally, we have
    \begin{equation}
      \begin{split}
      e_{trn} & = \frac{1}{N}\sum_{i = 1}^{N}\llbracket\argmax_{k}\xi^{trn}_k(\x_i) \ne
      l_i\rrbracket\\
      & = \frac{1}{N}\sum_{i = 1}^{N}\llbracket\argmax_{k}\xi^{loo}_k(\x_i)
      \ne l_i\rrbracket \\
      & = e_{loo}. 
      \end{split}
      \label{equ:error_equal}
    \end{equation}
  \end{proof}

  In practice, Lemma~\ref{lemma:1} and Equation~\eqref{equ:similar_alpha} could be only
  considered as approximately true. However, we still can treat the training error as a
  good estimate to the leave-one-out error. Recall that $n$-fold cross-validation is an
  approximation to the leave-one-out validation. Thus the cross-validation error, a
  commonly used criterion for model-selection, is also a estimate to the leave-one-out
  error. We then can directly employ the training error of ORE to choose the
  model-parameter, without an extra validation procedure. The fast model-selection, termed
  ``training-determined model-selection'' is justified empirically in the experiment. We
  tune the $\lambda$ for both ORE and Boosted-ORE, which is introduced below. No significant
  overfitting is observed. 
  
  \begin{corollary}
    One can directly set the parameter $\lambda$ to a very small value, \eg $\lambda =
    1e$-$5$, to achieve the ORE-model with a good generalization capability.
  \end{corollary}
  \begin{proof}
    Because the training error of a ORE-Learning directly reflects its generalization
    capability, the ORE-Learning is very resistant to overfittings. Considering that the
    main reason for imposing the regularization is to curb overfittings, one can totally
    discard the regularization term in optimization problem \eqref{equ:opt}. However, to
    prevent the problem from being ill-posed, we still need a constraint for $\Alpha$.
    Thus one $\lambda$ with a small value, which implies trivial regularizing effect, is
    appropriate.
  \end{proof}

  The above corollary is also verified in the experimental part. Admittedly, without an
  effective $\ell_1$ regularization, one can not expect the obtained ORE-model is sparse
  and efficient. Consequently, we still conduct the training-determined model-selection to
  strike the balance between accuracy and efficiency.

\section{ORE-Boosting for immense BPR sets}
\label{sec:boost_ore}

  In principle, the convex optimization for the ORE-Learning could be solved perfectly.
  Nonetheless, sometimes the patch number $T$ is enormous or even nearly infinite. In
  those scenarios, to solve problem \eqref{equ:opt} via normal convex solvers is
  impossible. Recall that boosting-like algorithms can exploit the infinite functional
  space effectively \cite{Schapire_ML_99_Improved, Friedman_AS_00_Special}. We therefore
  can solve the immense problem in a boosting fashion, \ie the BPRs are added into the
  ORE-model one by one, based upon certain criteria. 

\subsection{Solve the immense optimization problem via the column-generation}
\label{subsec:cg}

  The conventional boosting algorithms \cite{Schapire_ML_99_Improved,
  Friedman_AS_00_Special} conduct the optimization in a coordinate-descend manner.
  However, it is slow and can not guarantee the global-optimality at every step. Recently,
  several boosting algorithms based on the column-generation \cite{Demiriz_02_ML_LPBoost,
  Shen_PAMI_10_Dual, Hao_11_ACCV_Totally} were proposed and showed higher training
  efficiencies. We thus follow their principle to solve our problem. 

  To achieve the boosting-style ORE-Learning, the dual problem of \eqref{equ:opt} need to
  be derived firstly.

  \begin{theorem}
    The Lagrange dual problem of \eqref{equ:opt} writes
    \begin{equation}
      \begin{split}
        \min_{\u, r}~& r + \frac{1}{\lambda}\sum_i^N \left(u_i\log u_i - u_i\right) \\
        \sst~& \sum_{i = 1}^N u_i \left(b_{t,{l_i}}(\x_i) - \frac{1}{K}\right) \le r, ~ ~
        \forall t, \\
        ~& \u \psd 0.
      \end{split}
      \label{equ:dual_opt}
    \end{equation}
    \label{theo:dual}
  \end{theorem}
  \begin{proof}
    Firstly, let us rewrite the primal problem \eqref{equ:opt} as 
    \begin{equation}
      \begin{split}
        \min_{\Alpha}~&\sum_i^N\exp(\varphi_i) \\
        \sst~ & \varphi_i = -\sum_{t = 1}^T \alpha_t \left(b_{t,{l_i}}(\x_i) -
        \frac{1}{K}\right), ~ ~\forall i, \\
        ~& \Alpha \psd \mathbf{0}, ~ ~\|\Alpha\|_1 \le \lambda.
      \end{split}
      \label{equ:rewrite_opt}
    \end{equation}
    After assigning the \emph{Lagrange multipliers} \cite{Boyd_04_Convex} $\u \in
    \mathr^{N}$, $\q \in \mathr^{T}$ and $r \in \mathr$ associated with above constraints,
    we get the Lagrangian
    \begin{equation}
      \resizebox{0.43\textwidth}{!}{$
      \begin{split}
      L(\Alpha, \Varphi, \u, \q, r) = & \sum_i^{N}\exp(\varphi_i) -
      \sum_i^{N}u_i\left(\varphi_i + \sum_{t = 1}^{T}\alpha_t\theta_{t,i}\right) \\
      & -  \q^{\T}\Alpha + r(\boldsymbol{1}^{T}\Alpha - \lambda).
      \end{split}
      $}
      \label{equ:lagrangian}
    \end{equation}
    where $\theta_{t,i} = b_{t,{l_i}}(\x_i) - 1/K,~\forall t, i$ and $\q \psd
    \boldsymbol{0}$. The Lagrange dual function is defined as the ``infimum'' of the
    Lagrangian, \ie
    \begin{equation}
      \begin{split}
      \inf_{\Alpha, \Varphi}L = & \inf_{\Varphi}\left( \sum_{i =
          1}^{N}\exp(\varphi_i) - u_i\varphi_i\right)  - r\lambda \\
          & - \overbrace{ \left(\sum_i^{N}u_i\Theta^{\T}_i + \q^{\T} -
          r\boldsymbol{1}^{\T}\right)}^{\text{must be $\boldsymbol 0$}}\Alpha \\
          = & - \sum_{i = 1}^{N}\overbrace{\sup_{\varphi_i}(u_i\varphi_i -
          \exp(\varphi_i))}^{\text{the conjugate of} \exp(\varphi_i)} - r\lambda \\
          = & - \sum_{i = 1}^N(u_i\log u_i - u_i) - r\lambda
      \end{split}
      \label{equ:infimum}
    \end{equation}
    where $\Theta_i = [\theta_{1, i}, \theta_{2, i}, \cdots, \theta_{T, i}]$. After
    eliminating $\q$ we get the first $t$ constraints in the dual problem. The conjugate
    of function $\exp(\varphi_i)$ requires that $\u \psd \boldsymbol{0}$, \aka the second
    constraint of \eqref{equ:dual_opt}. The dual problem is to maximize the above
    Lagrangian. After simple algebraic manipulations, \eqref{equ:dual_opt} is obtained.
  \end{proof}
  In Theorem~\ref{theo:dual}, $\u = [u_1, u_2, \cdots, u_N]$ is usually viewed as the
  weighted data distribution. Considering that BPR is instanced-based and thus depends on
  $\u$, we then use $\b^{\u}_t$ to represent the $t$th BPR under the data distribution
  $\u$.

  With the column-generation scheme employed in \cite{Demiriz_02_ML_LPBoost,
  Shen_PAMI_10_Dual, Hao_11_ACCV_Totally} and Theorem~\ref{theo:dual}, we design a
  boosting-style ORE-Learning algorithm. The algorithm, termed ORE-Boosting, is summarized
  in Algorithm~\ref{alg:ore_boosting}. 

  \begin{algorithm}[ht]
  \caption{ORE-Boosting}
      \KwIn
      {
        \begin{itemize}
          \item A set of training data $\X = [\x_1, \x_2, \cdots, \x_N]$.
          \item A set of patch-locations, indexed by $1, 2, \cdots, T$.
          \item A termination threshold $\epsilon > 0$.
          \item A maximum training step $S$.
          \item A primitive dual problem:
            \begin{equation*}
                \min_{\u, r}~ r + \frac{1}{\lambda}\sum_i^N \left(u_i\log u_i -
                u_i\right),~ \sst ~\u \psd 0.
            \end{equation*}
        \end{itemize}
      }
      \Begin
      {
        $\cdot$ Initialize $\boldsymbol{\alpha} = 0$, $t = 0$, $u_i = 1/N, ~\forall i$; \\
        \For{$s \leftarrow 1$ \KwTo $S$}
        {
          $\cdot$ Find a new BPR, $\b^{\u}_{t^{\ast}}$, such that 
          \begin{equation}
            t^{\ast} =  \argmax_{t \in \{1, 2, \cdots, T\}}\sum_{i = 1}^N u_i
            \left(b^{\u}_{t,{l_i}}(\x_i) - 1/K\right);  
            \label{equ:best_bpr}
          \end{equation}
          $\cdot$ {\bf if } $\sum_{i = 1}^N u_i \left(b^{\u}_{t^{\ast},{l_i}}(\x_i) -
          1/K\right) < r + \epsilon$, {\bf break}; \\
          $\cdot$ Assign the inequality
          \[\sum_{i = 1}^N u_i \left(b^{\u}_{t^{\ast},{l_i}}(\x_i) - 1/K\right) \le r\] 
            into the dual problem as its $s$th constraint; \\
          $\cdot$ Solve the updated problem;
        }
        $\cdot$ Calculate the primal variable $\boldsymbol{\alpha}$ according to the
                dual solutions and KKT conditions;
      }
      \KwOut
      {
        The Boosted-ORE: 
        $\bxi(\y) = \argmax_k\textstyle\sum_{t = 1}^{T}\alpha_t\cdot\b_t(\y)$.\\
      }
  \label{alg:ore_boosting} 
  \end{algorithm}

\subsection{Ultrafast --- the data-weight-free training}
\label{subsec:fast_training}

  For the conventional basic hypotheses used in boosting, such as decision trees, decision
  stumps and the linear-LDA-classifiers, one needs to re-train them after the training
  samples' weights $\u$ are updated. Usually, the re-training procedure dominates the
  computational complexity \cite{Demiriz_02_ML_LPBoost, Shen_PAMI_10_Dual}.
  
  Apparently, we need to follow this computationally expensive scheme since BPRs are
  totally data-dependent. It is easy to see the computation complexity of each BPR is 
  \begin{equation}
    C_{L} = \mathcal{O}(M^3) + \mathcal{O}(M^2 d), 
   \label{equ:bpr_complex}
  \end{equation}
  then the complexity of the training procedure is given by 
  \begin{equation}
    C_{train} = T \cdot S \cdot C_{L} = \mathcal{O}(T S M^3) + \mathcal{O}(T S M^2 d), 
    \label{equ:train_complex}
  \end{equation}
  The whole training procedure could be very slow when $T$ and $S$ are both large.
  
  However, we argue that: {\em the ORE-Boosting can be performed much faster}. To explain
  this, let us firstly rewrite the constraint $\u \psd \boldsymbol{0}$ in
  \eqref{equ:dual_opt} as $\u \pd \boldsymbol{0}$. This change won't influence the
  interior-point-based optimization method \cite{Boyd_04_Convex}. Then we can prove the
  following theorem.

  \begin{theorem}
    Given that $\u \pd 0$, The BPRs are independent of the weight vector $\u$. In other
    words, for ORE-Boosting, all the BPRs need to be trained only once. 
  \label{theo:fast_train}
  \end{theorem}

  \begin{proof}
    let $\U_k \in \mathbb{R}^{M \times M}$ be the diagonal matrix such that $\U_k(i,i) =
    \u_k(i), ~i = 1, 2, \dots, M$, where $\u_k$ is the weight vector for the training
    face images from the $k$th class. By taking account of the data weight, the
    representation coefficients associated with patch $t$ are given by.
    \begin{equation}
      \hat{\Beta}^{\ast}_{t,k} = \argmin_{\Beta} ~ \|\y_t - \X^t_k \U_k \Beta\|_2, 
      \label{equ:w_bpr_ls}
    \end{equation}
    which has a closed-form solution that writes 
    \begin{equation}
      \hat{\Beta}^{\ast}_{t,k} = (\U_k{\X^t_k}^{\T}\X^t_k\U_k)^{-1}\U_k{\X^t_k}^{\T}\y_t
      \label{equ:deduction_1}
    \end{equation}
    and we know that 
    \begin{equation}
      \u \pd 0 \implies u_k > 0 \implies \U_k^{-1} ~\text{exists.}
      \label{equ:deduction_2}
    \end{equation}
    Thus \eqref{equ:deduction_1} can be further rewritten into 
    \begin{equation}
      \begin{split}
        \hat\Beta^{\ast}_{t,k} & =
        \U_k^{-1}{({\X^t_k}^{\T}\X^t_k)}^{-1}\U_k^{-1}\U_k{\X^t_k}^{\T}\y_t \\
        & = \U_k^{-1}{({\X^t_k}^{\T}\X^t_k)}^{-1}{\X^t_k}^{\T}\y_t \\
        & = \U_k^{-1}\Beta^{\ast}_{t,k},
      \end{split}
      \label{equ:deduction_3}
    \end{equation}
    where $\Beta^{\ast}_{t,k}$ is the solution to the unweighted BPR. We now can obtain
    the reconstruction residual $\hat{r}^t_k$ as
    \begin{equation}
      \begin{split}
        \hat{r}^t_k & = \|\y_t - \X^t_k\U_k\hat\Beta^{\ast}_{t,k}\|_2 \\
        & = \|\y_t - \X^t_k\U_k\U_k^{-1}\Beta^{\ast}_{t,k}\|_2 = r_{t,k}. \\
      \end{split}
      \label{equ:deduction_4}
    \end{equation}
    This result, without loss of generality, is valid for all the classes and patches.
    Considering that BPRs are determined by the associated residuals, we arrive at 
    \begin{equation}
      \b^{\u}_{t,k} = \b_{t,k}.
      \label{equ:deduction_5}
    \end{equation}
    That is to say, training data weights do not have any impact on the BPRs.

    Actually, a more intuitive understanding of the above analysis is in
    \eqref{equ:w_bpr_ls}: if we treat $\U_k\Beta$ as the variable of interest, we solve
    exactly the same problem as the standard least squares fitting problem. 
  \end{proof}
 
  According to the theorem, one needs to calculate the BPRs only once. In practice, the
  following calculations are conducted for all the BPRs and training samples.
  \begin{equation}
    c^t_i = b_{t, l_i}(\x_i) - \frac{1}{K}, ~ ~\forall t, i.
    \label{equ:pre_confidence}
  \end{equation}
  Note that $l_i$ is the ground-truth category of $\x_i$. $T$ oracle vectors $\c_t =
  [c^t_1, c^t_2, \cdots, c^t_N]^{\T},~ \forall t$ are stored beforehand. When we
  performing the ORE-Boosting, the optimization task in \eqref{equ:best_bpr} is reduced to 
  \begin{equation}
    t^{\ast} =  \argmax_{\forall t}\left(\c_t^{\T}\u\right);  
    \label{equ:simple_best_bpr}
  \end{equation}

  With the oracle vectors $\c_t,~\forall t$, the training cost is reduced by $S$ times to 
  \begin{equation}
      \tilde{C}_{train} = T \cdot C_{L} = \mathcal{O}(T M^3) + \mathcal{O}(T M^2 d).
    \label{equ:train_complex_fast}
  \end{equation}
  Usually, $S$ is of order $10^2$, so the above strategy can gain a speedup of a few
  hundred times (see Section~\ref{subsec:fast_training}). This desirable property makes
  the proposed ORE-Boosting very compelling in terms of computation efficiency.   

\section{Face Recognition Using ORE --- a Sophisticated Mixture of Inference and Learning}
\label{sec:strategy}

  As we discussed above, the LR-based algorithms are generative rather than
  discriminative. Their main goal is to reconstruct the test face using training faces.
  From another point of view, every LR-based algorithms is a pure inference procedure
  using the generative model associated with a specific linear-subspace-assumption. There
  is no learning process performed because the generative model is predetermined by the
  theoretical analysis \cite{Georghiades_PAMI_01_Few, Basri_PAMI_2003_Lambertian}.
  However, the theoretical proof is only valid under certain ideal conditions and the
  linear-subspace-assumption itself is an approximation to the derived \emph{illumination
  cone} \cite{Georghiades_PAMI_01_Few}. When this approximated model is applied with
  ``imperfect'' gallery faces, accuracy reductions always occur. 
  
  The ``imperfect'' training faces, from another perspective, usually imply more
  information or patterns involved. By effectively learning the meaningful ones, \eg
  expressions and disguises, we can enhance the prior-knowledge-determined model and
  achieve higher performance. Of course, not all the patterns can be included in the
  training set. When novel patterns arise during the test, one can only reduce their
  influence via a inferring process. In this sense, we argue that an ideal face-recognizer
  should contain two functional parts: 

  \begin{enumerate}
    \item A learner, which can extract the existing patterns from the training set.
    \item An inference approach, which can recognize the known patterns while discard the
      foreign ones in the test face.
  \end{enumerate}
  
  The learning algorithms for ORE models has been proposed in
  Section~\ref{sec:combine_bpr} and \ref{sec:boost_ore}. Now we design the ORE-tailored
  inference approach. 

\subsection{Robust-BPR -- BPR with a Generic-Face-Confidence}
\label{subsec:robust_bpr}

  The BPR is informative enough to describe a patch-based LR and one can learn certain
  face patterns within the ensemble learning framework. However, in the test phase, some
  unknown patterns, which usually present as non-face patches, might occur. Most
  LR-methods, including the standard BPR, only pay attention to distinguish the face
  between different individuals thus can hardly handle this kind of patterns. On the other
  hand, several evidences \cite{Meytlis_PAMI_07_Dimensionality} suggests that
  \emph{generic faces, including all the categories, also form a linear-subspace}. The
  linear-subspace is sufficiently compact comparing with the general image space.
  Furthermore, some visual tracking algorithms have already employed LR-approaches (SRC or
  its variations) to distinguish the foreground from the background
  \cite{Xue_ICCV_09_Track, Li_CVPR_11_Track}. 
  
  Inspired by the successful implementations, we propose to employ the linear
  representation for distinguishing face patches from face-unrelated or partly-face
  patches. Specifically, a badly-contaminated face patch is supposed to be distant from
  the linear subspace spanned by the training patches in the same position. In this
  manner, one can measure the degree of contamination for each test patch.
  
  \begin{figure}[h]
    \centering
    \includegraphics[width=0.58\textwidth]{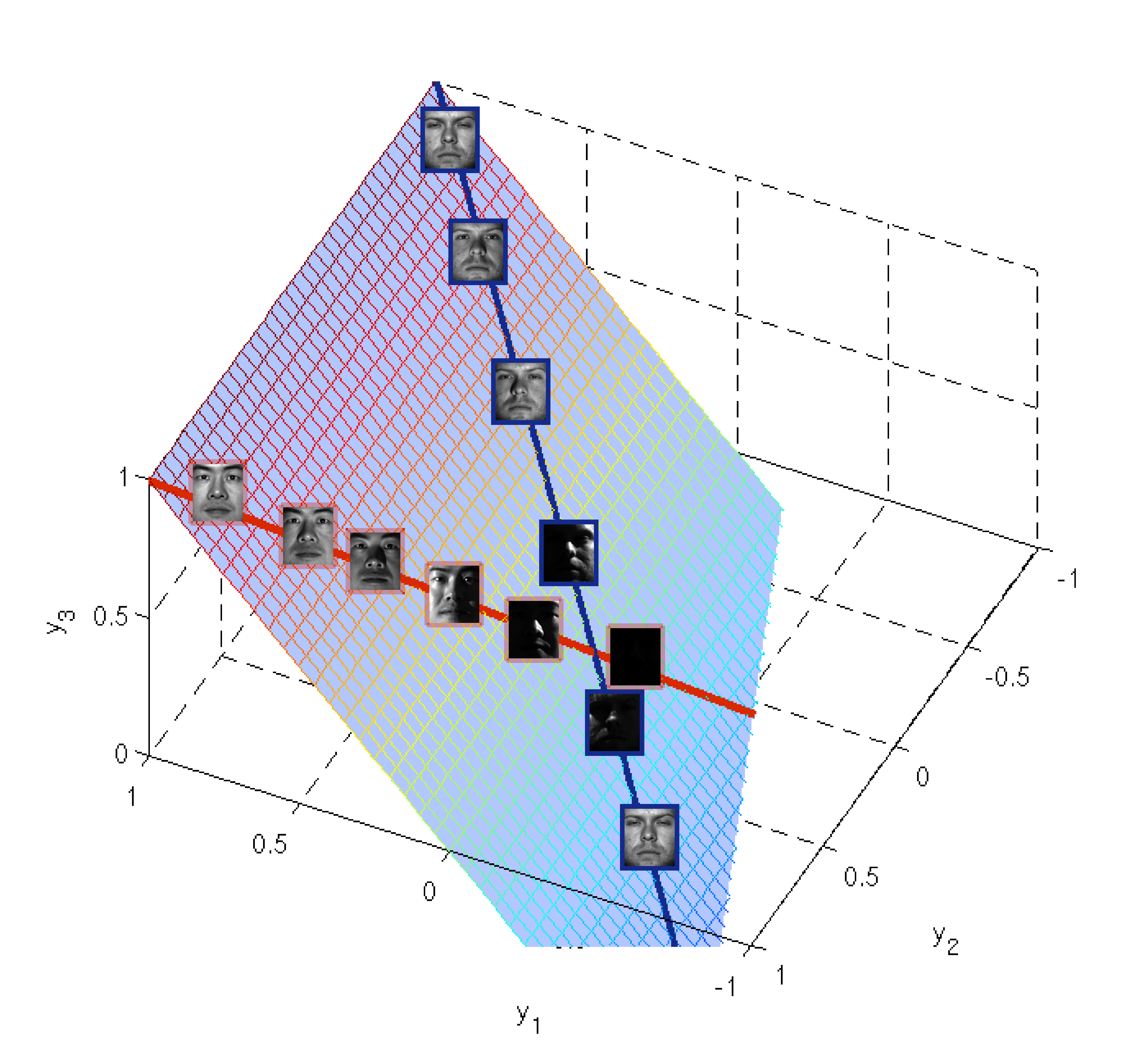}
    \caption
    {
      The demonstration of the generic-face subspace in the original $3$-D feature space.
      Faces from all the categories ($K = 2$ here) form a $2$-D linear-subspace, \ie a
      plane shown in light blue. Two linear-subspaces, \ie the lines shown in blue and red
      respectively, correspond to two different subjects. In this work, however, we are
      only interested in the face patches and consequently the ``generic-face-patch''
      subspace are considered instead. 
    }
   \label{fig:robust_bpr}
  \end{figure}
  
  Figure~\ref{fig:robust_bpr} illustrates the assumption about the linear-subspace of
  generic-faces. Note that the faces are merely for demonstration, in this paper, we
  actually focus on the face patches. According to this assumption, one test patch will be
  considered as a face part only when it is close enough to the corresponding
  ``generic-face-patch'' subspace. 

  Now we formalize this idea in the Bayesian framework. Given that all the training face
  patches $\X^t = [\X^t_1, \X^t_2, \cdots, \X^t_K] \in \mathr^{d \times N}$ are clean and
  forming the representation basis, for a test patch $\y_t$, the reconstruction residual
  $\tilde{r}^2_t$ is given by:
  \begin{equation}
    \tilde{r}^2_t = \|\y_t - \X^t({\X^t}^\T \X^t)^{-1}{\X^t}^\T \y_t\|_2.
    \label{equ:generic_residual}
  \end{equation}
  Let us use the notation $u_t = 1$ to indicate that $\y_t$ is a face patch while $u_t =
  0$ indicates the opposite. After taking the non-face category into consideration, the
  original posterior in \eqref{equ:post_bt} is equivalent to $\Prob{\gamma_\y = k~ |
  ~ u_t = 1, \y_t}$. The new target posterior becomes
  \begin{equation}
    \begin{split}
      \tilde{b}_{t,k} & = \Prob{\gamma_\y = k, u_t = 1 \mid \y_t} \\ 
                    & = \Prob{\gamma_\y = k  \mid  u_t = 1, \y_t} \cdot \Prob{u_t =
                    1 \mid \y_t} \\
                    & = b_{t,k} \cdot \Prob{u_t = 1 \mid \y_t}.
    \end{split}
    \label{equ:robust_post_bt}
  \end{equation}
  Following the principle of linear-subspace, we can assume that
  \begin{equation}
    \begin{split}
      \Prob{\y_t \mid u_t = 0} & = C_0 \\
      \Prob{\y_t \mid u_t = 1} & = C_1 \cdot \exp(-{{\tilde{r}^2_t}}/{\tilde{\delta}}),
    \end{split}
    \label{equ:condition_generic}
  \end{equation}
  where $C_1$, $C_0$ is the normalization constant. The subspace for the non-face category
  is the universe space $\mathr^d$, which leads to the uniform distribution $\Prob{\y_t
  \mid u_t = 0} = C_0$. Recall that all the patches are normalized, thus the domain of
  $\y_t$ is bounded. One can calculate both $C_1$ and $C_0$ with a specific
  $\tilde{\delta}$. For simplicity, let us define 
  \begin{equation}
    \tilde{C} = \frac{C_0\cdot\Prob{u_t = 0}}{C_1\cdot\Prob{u_t = 1}} = \frac{C_0}{C_1},
    \label{equ:robust_prior}
  \end{equation}
  because without any specific prior we usually consider $\Prob{u_t = 0} = \Prob{u_t =
  1}$. We then arrive at the new posterior, which is given by
  \begin{equation}
    \begin{split}
      \tilde{b}_{t,k} & = b_{t,k} \cdot \Prob{u_t = 1 \mid \y_t} \\
      & = \frac{b_{t,k} \cdot \Prob{\y_t \mid u_t = 1}\cdot \Prob{u_t =
                  1}}{\sum_{j \in \{0, 1\}}\Prob{\y_t \mid u_t = j}\cdot \Prob{u_t = j}} \\ 
                  & = \frac{b_{t,k}}{1 + \tilde{C}\exp({\tilde{r}^2_t}/\tilde{\delta})}.
    \end{split}
    \label{equ:generic_post_derive}
  \end{equation}
  In practice, we replace the original $\tilde{b}_{t,k}$ with its upper bound 
  \begin{equation}
    \frac{1}{\tilde{C}}\cdot \exp({-\tilde{r}^2_t}/\tilde{\delta}) \cdot b_{t,k}
    \label{equ:generic_upper_bound}
  \end{equation}
  Note that the constant $\tilde{C}$ won't influence the final classification result as
  all the BPRs are linear combined. As a result, we can discard the term $1/\tilde{C}$ and
  avoid the complex integral operation for calculating it.

  We call the term $\exp(-\tilde{r}^2_t/\tilde{\delta})$ the
  \emph{Generic-Face-Confidence} (GFC) as it peaks when the patch is perfectly represented
  by generic face patches.  With this confidence, we can easily estimate how an image
  patch is face related, or in other words, how is it contaminated by occlusions or
  noises. The BPR equipped with a GFC is less sensitive to occlusions and noises, so we
  refer $\tilde{\b}_{t,k} = [\tilde{b}_{t,1}, \tilde{b}_{t,2}, \cdots, \tilde{b}_{t,K}]$
  as the Robust-BPR. The variance $\tilde{\delta}$ is usually data-dependent, we set
  \begin{equation}
    \tilde{\delta} = 0.05\cdot\left(\frac{1}{T}\sum_t^T\tilde{r}_t\right)^2, 
    \label{equ:prac_var_robust}
  \end{equation}
  for all the faces.

\subsection{The GFC-equipped inference approach}
\label{subsec:inference}

  With the unknown patterns, the learned patch-weights $\alpha_t,~\forall t$ could not
  guarantee their optimality anymore. An highly-weighted patch-location could be corrupted
  badly on the test image. Consequently, it should merely play a trivial role in the test
  phase. In other words, the importances of all the patches should be reevaluated. We then
  employ the proposed GFC to amend the importances for each patch. When the test face is
  possibly contaminated, we aggregate the Robust-BPRs instead of the original BPRs. In
  addition, the learned $\alpha_t,~\forall t$ are not as reliable as before thus we
  replace the original $\alpha_t$ with its ``faded'' version, \ie
  \begin{equation}
    \tilde{\alpha}_t = \alpha_t^q, ~ ~ q \in [0, 1], ~\forall t, 
    \label{equ:faded_alpha}
  \end{equation}
  where $q$ is the ``fading coefficient''. The smaller the $q$ is, the less we take
  account of the learned weights. Now we arrive at the new aggregation, which writes: 
  \begin{equation}
    \bxi = \sum_{t = 1}^{T'}\alpha_t\b_t ~\longrightarrow~ \tilde{\bxi} = \sum_{t =
    1}^{T'}\alpha_t^q\text{GFC}_t\b_t
    \label{equ:robust_ore}
  \end{equation}
  where $T'$ is the number of selected patches via the previous ORE-Learning and usually
  $T' \ll T$ as we impose a $\ell_1$ regularization on the loss function.
  Figure~\ref{fig:stems} gives us a explicit illustration of mechanism of the patch-weight
  amending procedure. In the upper row, $31$ patches are selected by using ORE-Learning.
  Their weights are also shown as stems in the left chart. When a test face is badly
  contaminated by noisy occlusions, as shown in the bottom row, those weights are not
  reliable anymore. After modified by the proposed methods, all the large weights are
  assigned to the clean locations. Consequently, the following classification can hardly
  influenced by the occlusions. \emph{From a bionic angle, the weight-amendment is
  analogue to a focus-changing procedure, as the previously emphasized parts look
  ``unfamiliar'' and not reliable anymore.}

  \begin{figure}[ht!]
    \begin{center}$
    \begin{array}{lc}
    \subfigure{\label{subfig:ori}\includegraphics[width=0.2\textwidth]{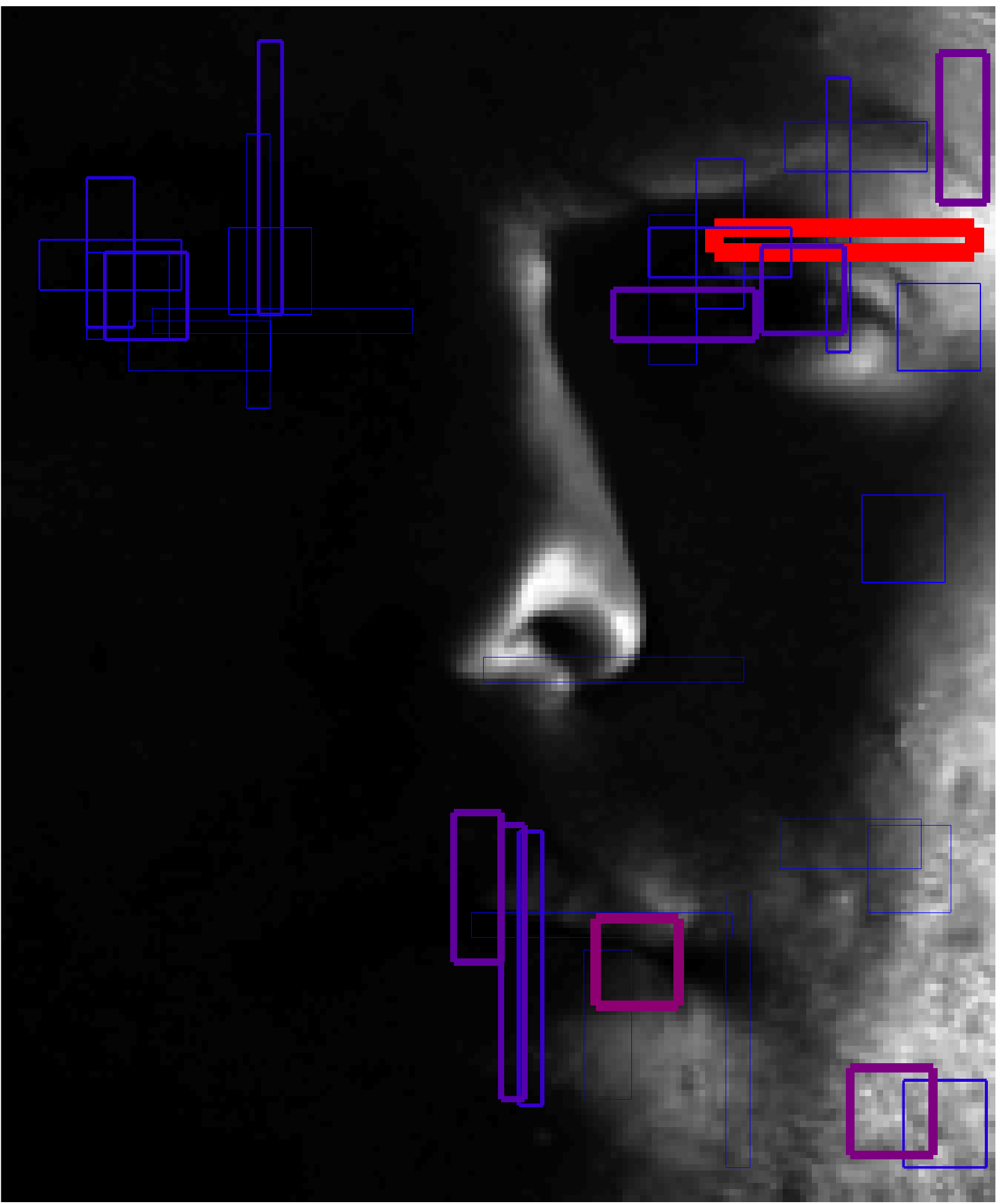}}
    & 
    \subfigure{\label{subfig:ori_stem}\includegraphics[width=0.55\textwidth]{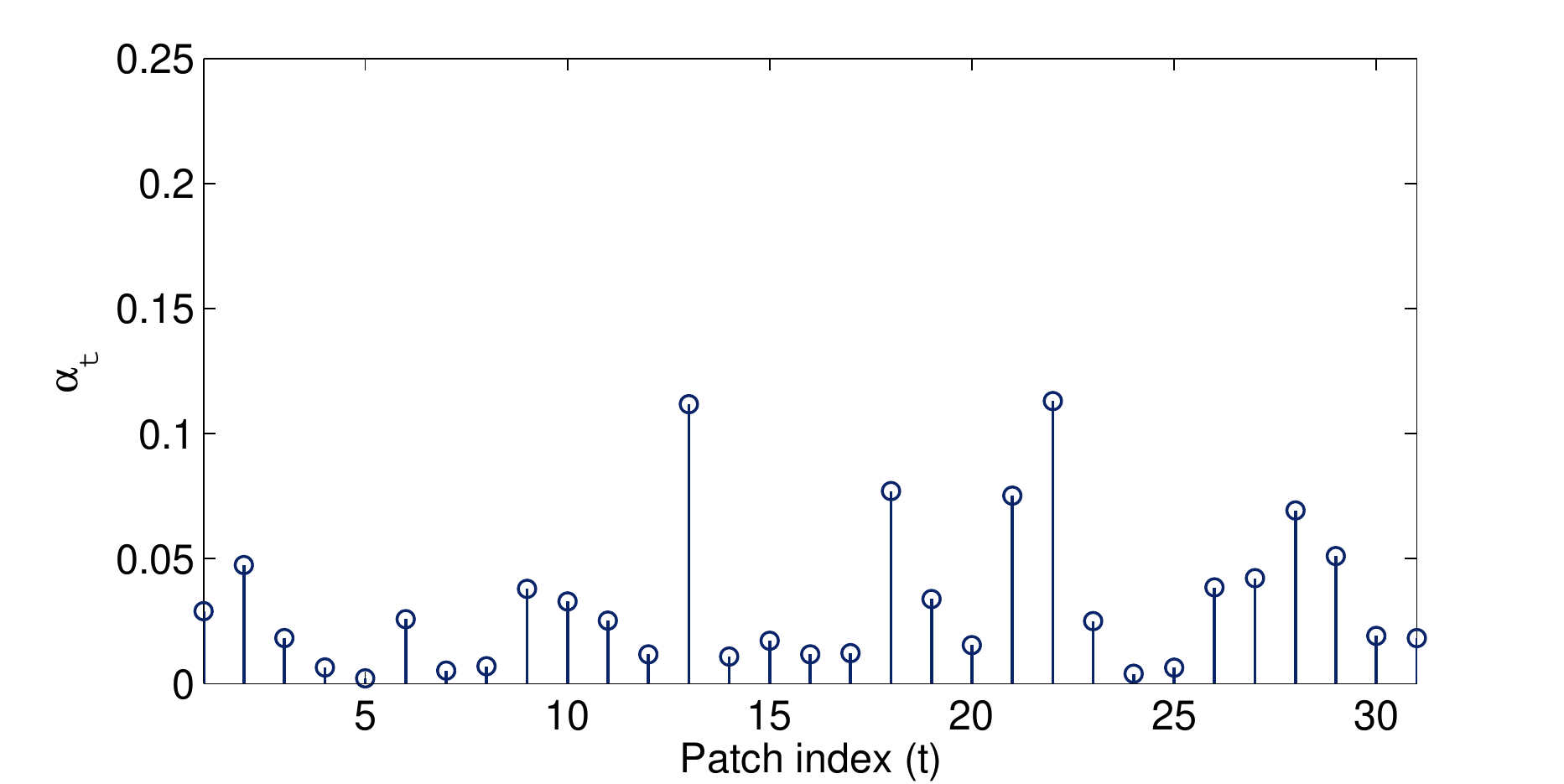}}
    \\
    \subfigure{\label{subfig:gfc}\includegraphics[width=0.2\textwidth]{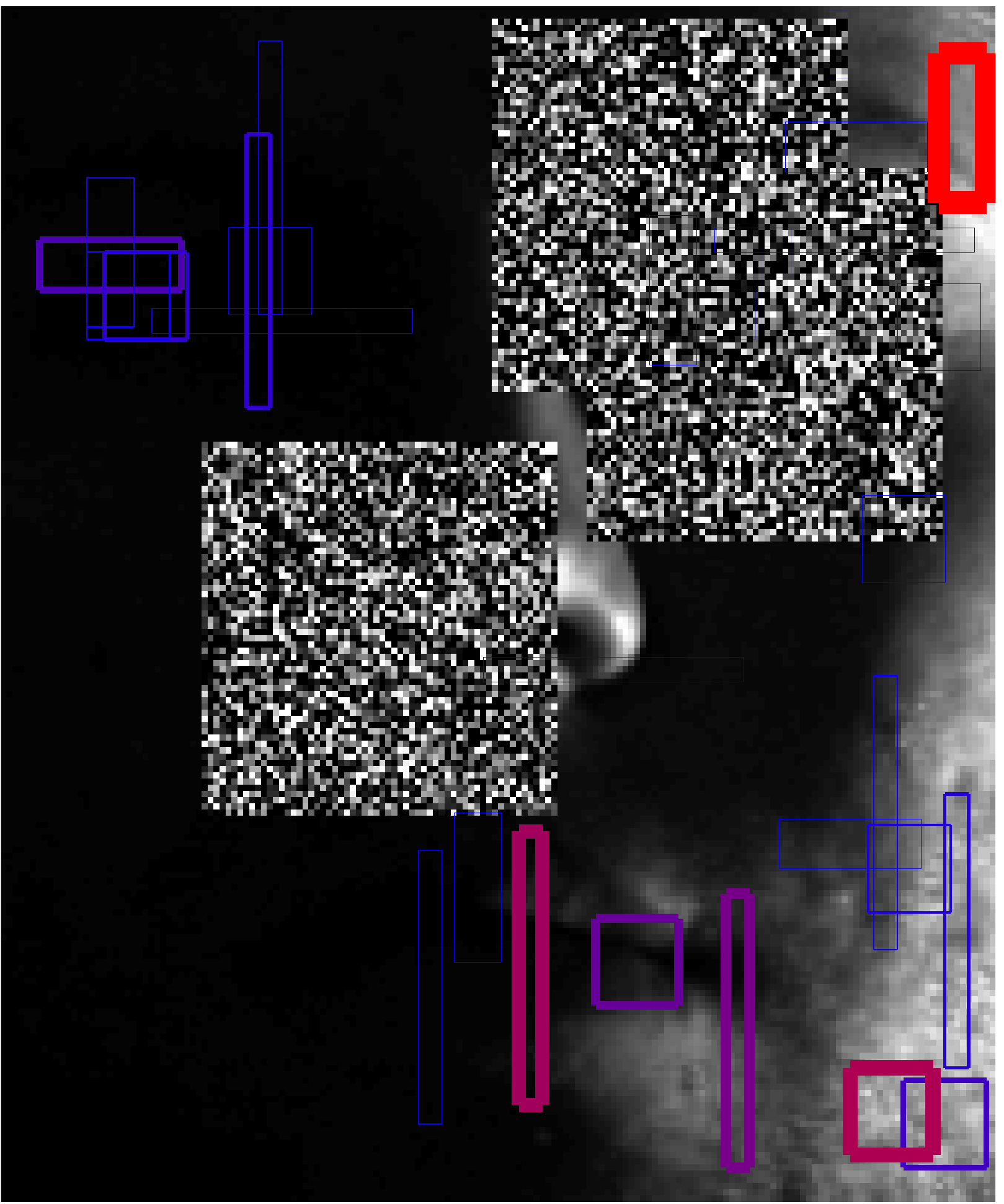}}
    & 
    \subfigure{\label{subfig:gfc_stem}\includegraphics[width=0.55\textwidth]{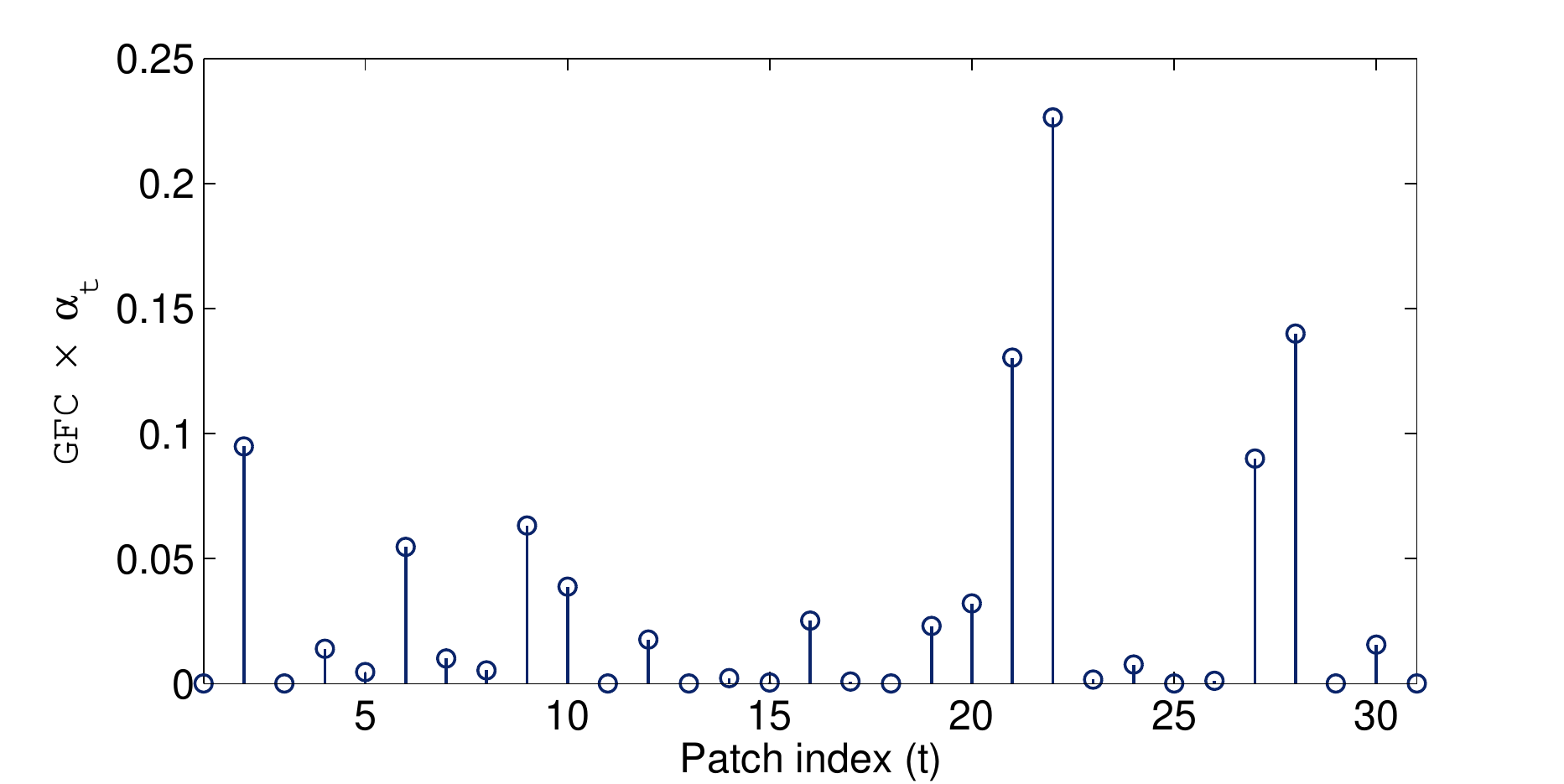}}
    \end{array}$
    \end{center}
    \caption
    {
        The demonstration for the patch-weight amending method. Upper row: the selected
        patches by ORE-Learning. Their weights are shown as stems in the left chart.
        Bottom row: one test face contaminated by three noisy blocks. The patches' weights
        are modified by using GFC. We can observe that the ORE model with new weights pay
        more attention to the clean patches. In other words, the ``attention'' changes
        because some pre-trusted parts are not reliable anymore.
    }
   \label{fig:stems}
  \end{figure}

  The patch-weight amendment and the subsequent aggregation process compose the inference
  part of the ORE algorithm. To distinguish the inference-facilitated ORE-model from the
  original ones, we refer to it as Robust-ORE. Compared with the anti-noise method
  proposed in \cite{Wright_PAMI_09_Face} (see optimization
  problem~\eqref{equ:opt_robust_src}), our Robust-ORE does not impose a sparse assumption
  on the corrupted part thus we can handle much larger occlusions. Furthermore, our method
  is much faster than the robust SRC while maintains its high robustness, as shown in the
  experiment. Most recently, Zhou \etal \cite{Zhou_09_ICCV_MRF} proposed a advanced
  version of \eqref{equ:opt_robust_src} via imposing a spatially-continuous prior to the
  error vector $e$. The algorithm, admittedly, performed very well, especially on the face
  with single occlusion. However, we argue that the performance gain is due to the extra
  spatial prior knowledge. In this paper, none of the spatial relation is considered. 
  
  
\subsection{The learning-inference-mixed strategy}
\label{subsec:mixture}

  \begin{figure*}[h!]
    \centering
    \includegraphics[width=0.9\textwidth]{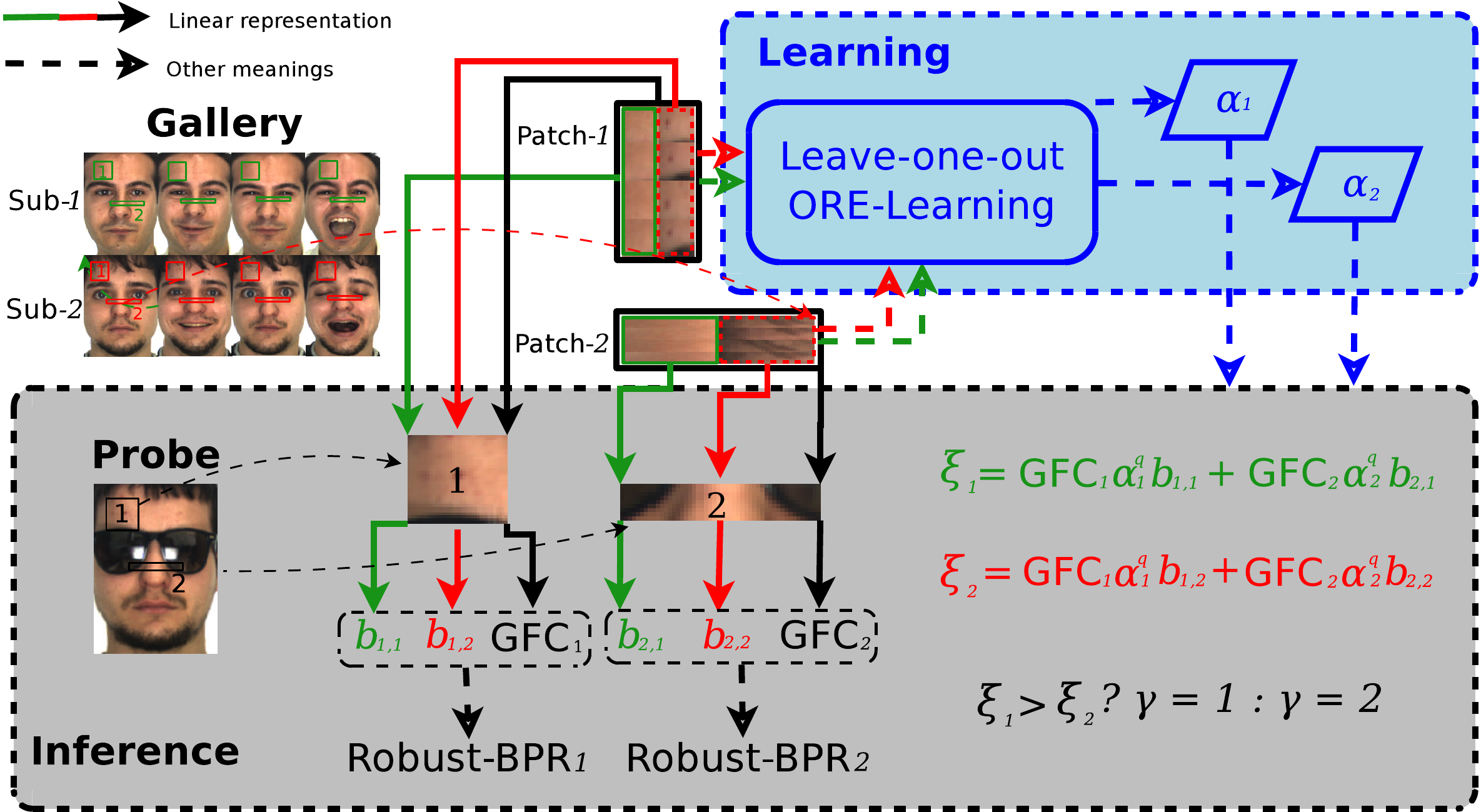}
    \caption
    {
      The demonstration of ORE-Learning and the inference procedure. The
      extremely-simplified problem only contains two subjects and two patch candidates.
      All the green items are related to Sub-$1$ while the red ones are related to
      Sub-$2$. The solid arrows indicate linear representation approaches, with different
      colors standing for different representation basis. The black solid arrows represent
      the representations based on the all the patches from a certain position while the
      green and red ones stand for those corresponding to Sub-$1$ and Sub-$2$
      respectively.
    }
   \label{fig:implementation}
  \end{figure*}

  Figure~\ref{fig:implementation} summarizes the ORE algorithm with a simplified setting
  where only two subjects (Sub-$1$ and Sub-$2$) and two patches (patch-$1$ on the right
  forehead and patch-$2$ on the middle face) are involved.

  From the flow chart, we can see that the ORE algorithm is, in essence, a sophisticated
  mixture of inference and learning. First of all, the patches are cropped and collected
  according to their locations and identities (different columns in one collection).
  Secondly, the leave-one-out margins are generated based on the leave-one-out BPRs. Then
  the existing face patterns are learned via the ORE-Learning or ORE-Boosting procedure.
  The learned results, $\alpha_1$ and $\alpha_2$, indicate the importances of the two
  patches. When a probe image is given, one perform $3$ different linear representations
  for each test patch. The LRs with the patches from Sub-$1$ and Sub-$2$ generate the BPRs
  $b_{t,1}$ and $b_{t,2}$ ($t \in \{1, 2\}$) respectively. In addition, we also use all
  the patches from one location to represent the corresponding test patch. In this way,
  the Generic-Face-Confidence (GFC$_t,~\forall t$) is calculated for each location. When
  calculating the ORE output $\xi_i,~i \in \{1, 2\}$, we multiply the term $\alpha_t^q
  b_{t,k}$ with the corresponding GFC$_t$. In this sense, one reduces the influence of
  unknown patterns (like the sunglasses in the example) arise in the test image. This is,
  typically, an inference manner based on the learned information ($\alpha_1$ and
  $\alpha_2$) and the prior assumption (the linear-subspaces corresponding to different
  individuals and the generic face patches). Finally, the identity $\gamma$ is obtained
  via a simple comparison operation.

  The excellence of this learning-inference-mixed strategy is demonstrated in the
  following experiment part.
\section{Experiments}
\label{sec:exp}

\subsection{Experiment setting}
\label{subsec:setting}

  We design a series of experiments for evaluating different aspects of the proposed
  algorithm on two well-known datasets, \aka Yale-B \cite{Georghiades_PAMI_02_Few} and AR
  \cite{Martinez_98_AR}. We compare the recognition rates, between the ORE algorithm and
  other LR-based state-of-the-arts methods, \ie Nearest Feature Line (NFL)
  \cite{Li_NN_99_NFL}, Sparse Representation Classification (SRC)
  \cite{Wright_PAMI_09_Face}, Linear Regression Classification (LRC)
  \cite{Nassem_PAMI_10_Face} and the two modular heuristics: DEF and Block-SRC. As a
  benchmark, the Nearest Neighbor (NN) algorithm is also performed. For the conventional
  LR-based methods, random projection (Randomfaces) \cite{Wright_PAMI_09_Face}, PCA
  (Eigenfaces) \cite{Turk_CVPR_91_Face} and LDA (Fisherfaces) \cite{Liu_TIP_02_Gabor} are
  used to reduce the dimensionality to $25$, $50$, $100$, $200$, $400$. Note that the
  dimensionality of Fisherfaces are constrained by the number of classes. 

  The ORE algorithm is performed using the patches each comprised of $225$ pixels. The
  widths of those patches are randomly selected from the set $\{5, 9, 15, 25, 45\}$ and
  consequently we generate the patches with $5$ different shapes. Random projections are
  employed to further reduce the dimensionality to $25$, $50$ and $100$. We treat the
  ORE's results with original patches ($225$-D) as its $200$-D performance. The inverse
  value of the trade-off parameter, \ie $\frac{1}{\lambda}$, is selected from candidates
  $\{10, 20, 30, 40, 50, 60, 70, 80, 90, 100\}$, via the training-determined model-selection
  procedure. The variance $\delta$ and $\tilde{\delta}$ are set according to
  \eqref{equ:prac_var} and \eqref{equ:prac_var_robust} respectively. We let $q = 0.2$ for
  Robust-ORE. As to ORE-Boosting, we set the convergence precision $\epsilon =
  1e\text{-}5$ and the maximum iteration number $S = 100$. 

  When carrying out the modular methods, we partition all the faces into $8$ ($4 \times
  2$) blocks and downsample each block to smaller ones in the size of $12 \times 9$, as
  recommended by the authors \cite{Nassem_PAMI_10_Face}. For a fair comparison, we also
  reduce the dimensionality of the face patches to $100$ using random mapping when
  performing the ORE algorithm.
  
  We conduct the test in different experimental settings to verify the recognition
  capacities of our method both in terms of inference and learning.  With each
  experimental setting, the test is repeated $5$ times and we report the average results
  and the corresponding standard deviations. Every training and test sample, \eg faces,
  patches and blocks, are normalized so that 
  \[\|\x_i\|_2 = \|\y\|_2 = 1,~\forall i.\]  
  All the algorithms are conducted in Matlab-R2009$a$, on the PC with a $2.6$GHz quad-core
  CPU and $8$GB RAM. When testing the running speed, we only enable one CPU-core. All the
  optimization, including the ones for ORE-Learning, ORE-Boosting and SRC, are performed
  by using Mosek \cite{Mosek}.

\subsection{Face recognition with illumination changes}
\label{subsec:fr_illumination}

  Yale-B contains $2,414$ well-aligned face images, belonging to $38$ individuals,
  captured under various lighting conditions , as illustrated in Figure~\ref{fig:YaleB}.
  For each subject, we randomly choose $30$ images to compose the training set and other
  $30$ images for testing. The Fisherfaces are only generated with dimensionality $25$ as
  LDA requiring that the reduced dimensionality is smaller than the class number. When
  performing LRC and ORE with $25$-D data, we only randomly chose $20$ training faces
  since the least-square-based approaches need an over-determined linear system. For this
  dataset, we employ $500$ random patches as the task is relatively easy.
  
  \begin{figure}[ht!]
    \centering
    \includegraphics[width=0.6\textwidth]{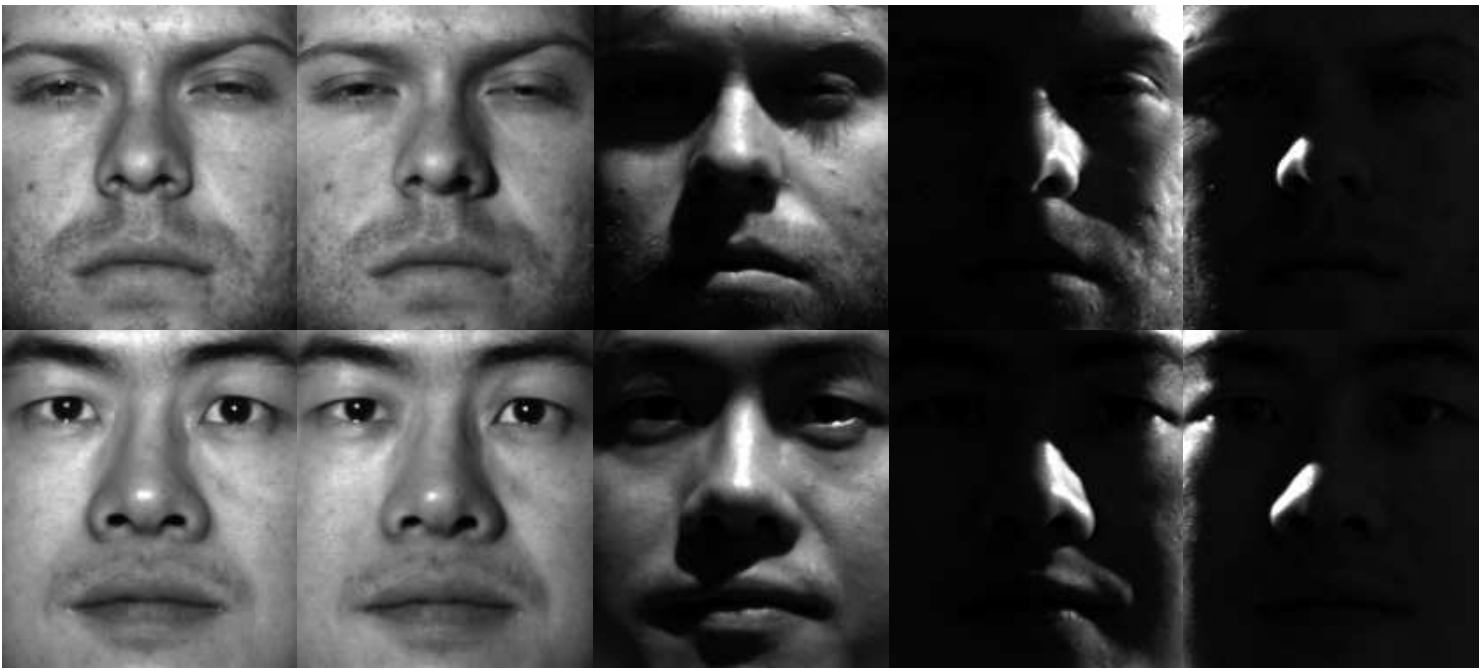}
    \caption
    {
      The demonstration of Yale-B dataset with extreme illumination conditions. 
    }
  \label{fig:YaleB}
  \end{figure}
  
    \begin{table*}[ht]
    \centering
    \resizebox{0.7\textwidth}{!}
    {
      \begin{tabular}{ l | l | l | l | l | l | l }

      \hline\hline
      &   & $25$-D    & $50$-D   & $100$-D    & $200$-D    & $400$-D   \\
      \cline{1-7}
      {\multirow{4}{*}{LDA}}  
      & NN        & $93.4\pm1.3$ & - & - & - & - \\
      & NFL       & $89.4\pm1.0$ & - & - & - & - \\
      & SRC       & $92.5\pm1.2$ & - & - & - & - \\
      & LRC       & $58.0\pm1.9$ & - & - & - & - \\
      \cline{1-7}
      {\multirow{4}{*}{Rand}}  
      & NN        & $42.6\pm4.0$ & $51.4\pm1.5$ & $54.2\pm3.0$ & $54.8\pm1.7$ & $56.6\pm1.5$ \\
      & NFL       & $83.2\pm1.7$ & $88.2\pm1.0$ & $89.5\pm0.6$ & $90.7\pm0.5$ & $90.9\pm0.4$ \\
      & SRC       & $80.1\pm1.6$ & $90.7\pm1.0$ & $94.7\pm0.5$ & $96.6\pm0.7$ & $97.1\pm0.5$ \\
      & LRC       & $25.9\pm4.1$ & $88.1\pm0.6$ & $93.1\pm1.2$ & $94.5\pm0.4$ & $94.7\pm0.4$ \\
      \cline{1-7}
      {\multirow{4}{*}{PCA}}  
      & NN        & $22.3\pm1.8$ & $30.4\pm1.7$ & $34.4\pm0.5$ & $36.6\pm1.2$ & $37.0\pm1.0$ \\
      & NFL       & $69.5\pm1.4$ & $77.4\pm1.2$ & $81.4\pm1.0$ & $83.0\pm0.5$ & $83.5\pm0.5$ \\
      & SRC       & $80.4\pm1.6$ & $89.1\pm0.9$ & $92.8\pm0.8$ & $94.2\pm0.7$ & $95.1\pm0.7$ \\
      & LRC       & $74.7\pm1.9$ & $88.1\pm0.4$ & $89.8\pm0.3$ & $90.7\pm0.5$ & $90.8\pm0.6$ \\
      \cline{1-7}
      \multicolumn{2}{l|}{{ORE}}        
      & $96.5\pm0.5$ & $99.6\pm0.2$ & $99.7\pm0.1$ & $\bf 99.9\pm0.1$ & - \\
      \multicolumn{2}{l|}{{Robust-ORE}}  
      & $\bf 98.3\pm0.3$ & $\bf 99.8\pm0.2$ & $\bf 99.9\pm0.1$ & $\bf 99.9\pm0.1$ & - \\
      \multicolumn{2}{l|}{{Boosted-ORE}}  
      & $95.6\pm1.2$ & $99.6\pm0.2$ & $99.8\pm0.1$ & $\bf 99.9\pm0.1$ & - \\
      \hline\hline
      \end{tabular}
    }
    \caption
    {
        The comparison of accuracy on Yale-B. The highest recognition rates are shown in
        bold. Note that we only perform algorithms with the Fisherface (LDA) on the $25$-D
        feature space. The original patch has $225$ pixels, thus we can't conduct ORE
        algorithms with $400$-D features.
    }
    \label{tab:yaleb_accuracy}
    \end{table*}


  The experiment results are reported in Table~\ref{tab:yaleb_accuracy}. As can be seen,
  the ORE-based algorithms \emph{consistently outperform all the competitors}.  Moreover,
  all the proposed methods achieve the accuracy of $99.9\%$ on $200$-D ($225$-D in fact)
  feature space. To our knowledge, {\em this is the highest recognition rate ever reported
  for Yale-B under similar circumstances}. Given $1,140$ faces are involved as test
  samples and the recognition rate $99.9\%$, \emph{only $1$ faces are incorrectly
  classified in average}. In particular, Robust-ORE, \ie ORE equipped with Robust-BPRs,
  shows the highest recognition ability.  Its recognition rates are always above $99.8\%$
  when $d \ge 50$.  The boosting-like variation of the ORE algorithm performs similarly to
  its prototype and also superior to the performances of other compared methods. 

    
  Figure~\ref{fig:boost_curve_yaleb} shows the boosting procedure, \ie the training and
  test error curves, for the ORE-Boosting algorithm with $100$-D features. We observe fast
  decreases for both curves. That justifies the efficacy of the proposed boosting
  approach. Furthermore, no overfitting is illustrated even though the optimal model
  parameter $\lambda$ is selected according to the training errors. It empirically
  supports our theoretical analysis in Section~\ref{subsec:training-determined}.  

  \begin{figure}[ht]
    \centering
    \includegraphics[width=0.68\textwidth]{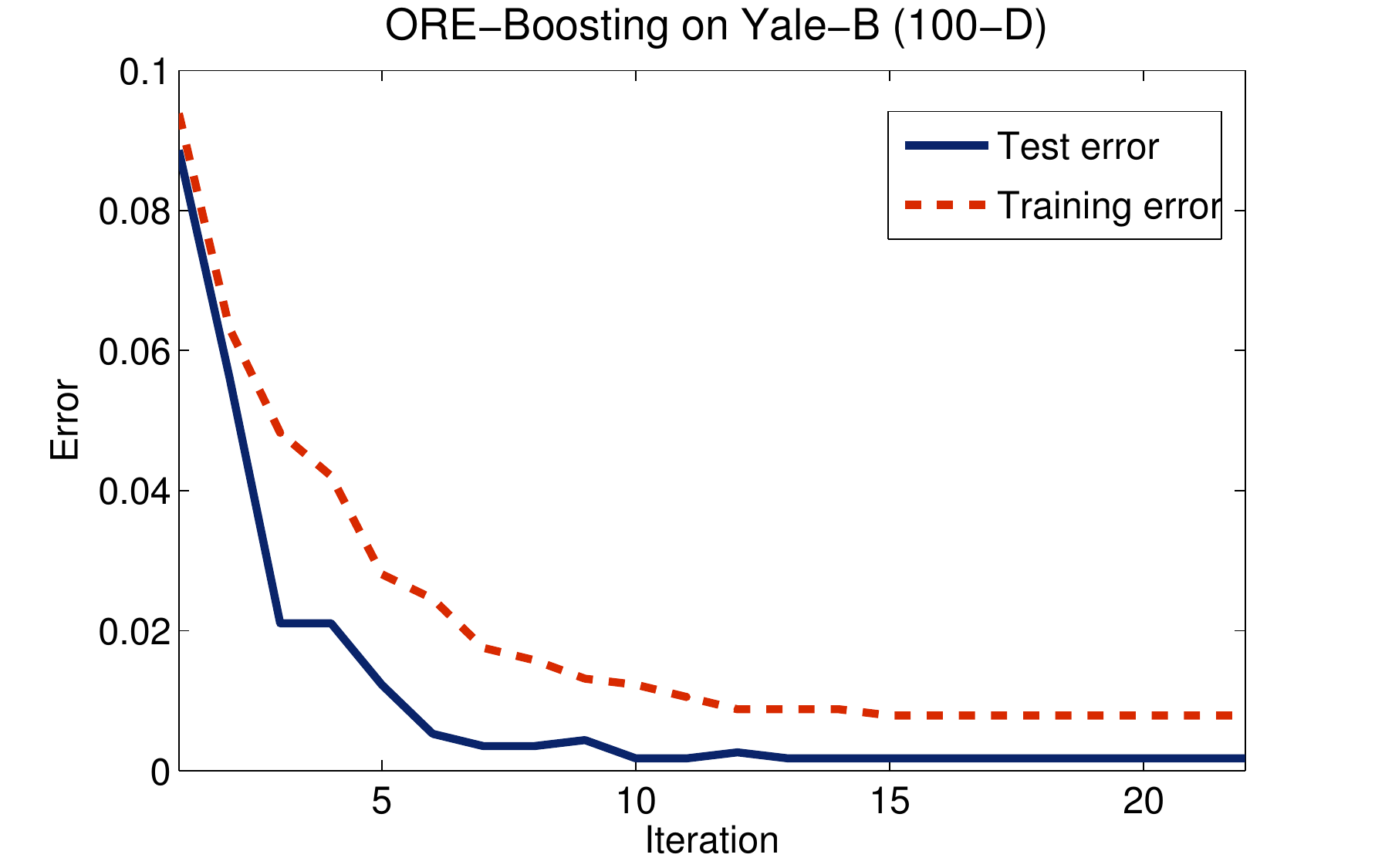}
    \caption[Plots of the boosting procedure of Yale-B]
    {
        Demonstration of the boosting procedure of ORE-Boosting with
        $100$-D features on Yale-B. 
    }
    \label{fig:boost_curve_yaleb}
  \end{figure}

  On Yale-B, both ORE-Learning and its boosting-like cousin only select a very limited
  part (usually around $5\%$) of all the candidate patches, thanks to the $\ell_1$-norm
  regularization. To illustrate this, Figure~\ref{fig:yaleb_selected_patches} shows all
  the candidates (Figure~\ref{subfig:yaleb_all_patches}), the selected patches by
  ORE-Learning (Figure~\ref{subfig:yaleb_ore_patches}), and those selected by ORE-Boosting
  (Figure~\ref{subfig:yaleb_boost_patches}). We can see that two algorithms make similar
  selections: in terms of patch positions and patch numbers ($32$ for ORE-Boosting \vs
  $31$ for ORE-Learning). Nonetheless, minor differences is shown \wrt the weight
  assignment, \ie assigning values to the coefficients $\alpha_i,~\forall i$. The
  ORE-Boosting aggressively assigns dominant weights to a few patches. In contrast,
  ORE-Learning distributes the weights more uniformly. The more conservative strategy
  often leads to a higher robustness. 

  \begin{figure}[ht!]
    \centering
    \subfigure[candidates]{\label{subfig:yaleb_all_patches}\includegraphics[width=0.25\textwidth]{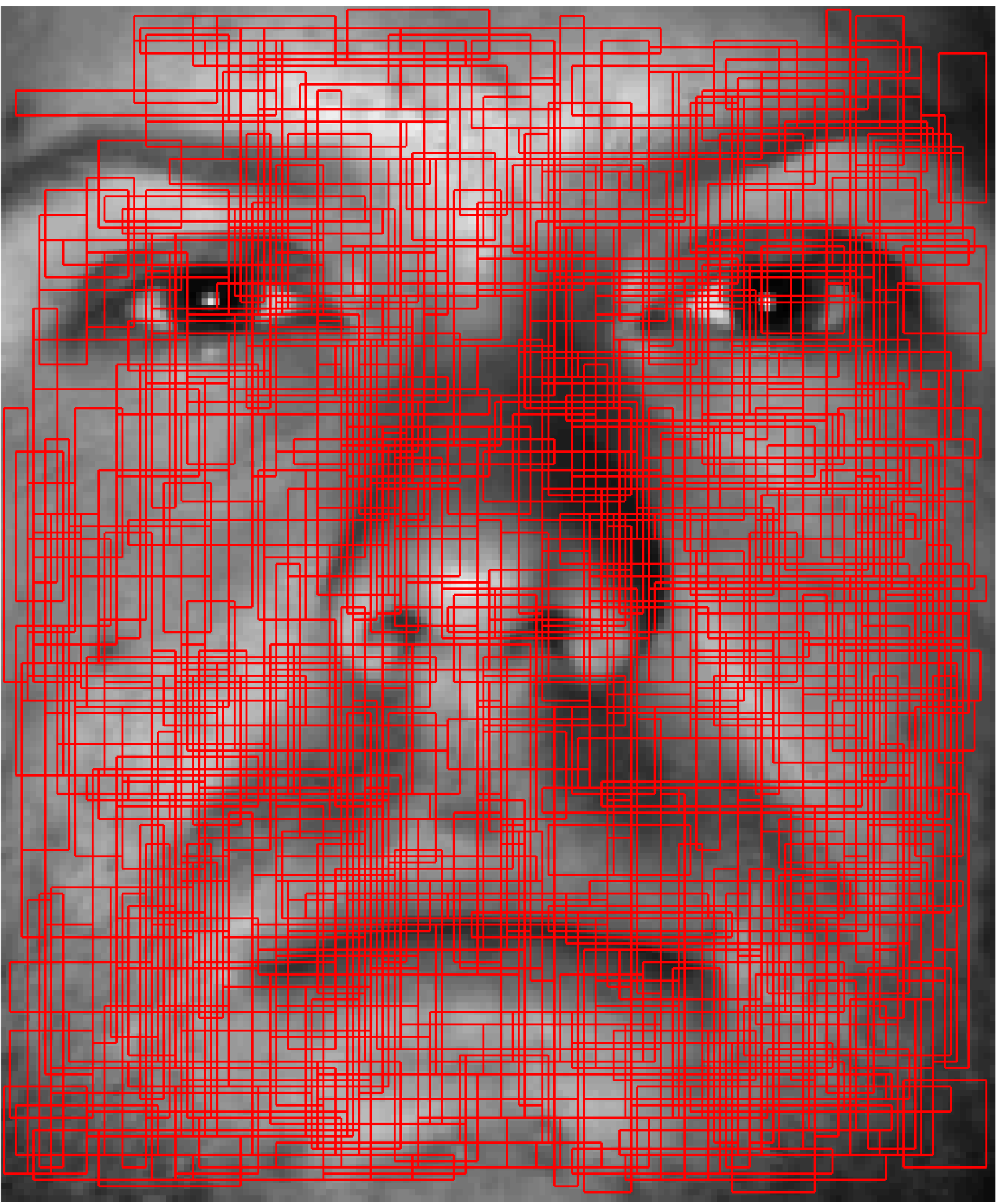}}
    \subfigure[ORE-Learning]{\label{subfig:yaleb_ore_patches}\includegraphics[width=0.25\textwidth]{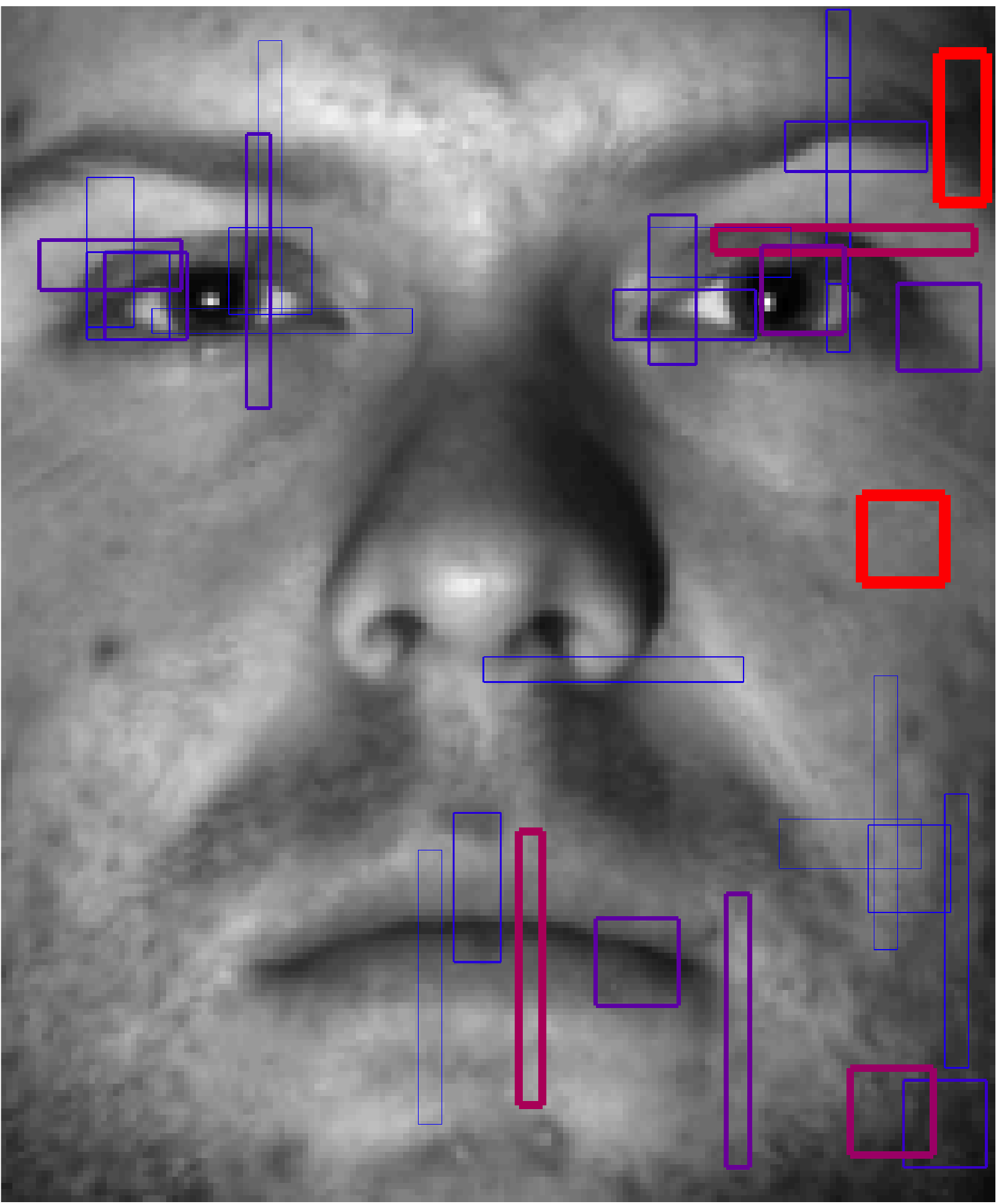}}
    \subfigure[ORE-Boosting]{\label{subfig:yaleb_boost_patches}\includegraphics[width=0.25\textwidth]{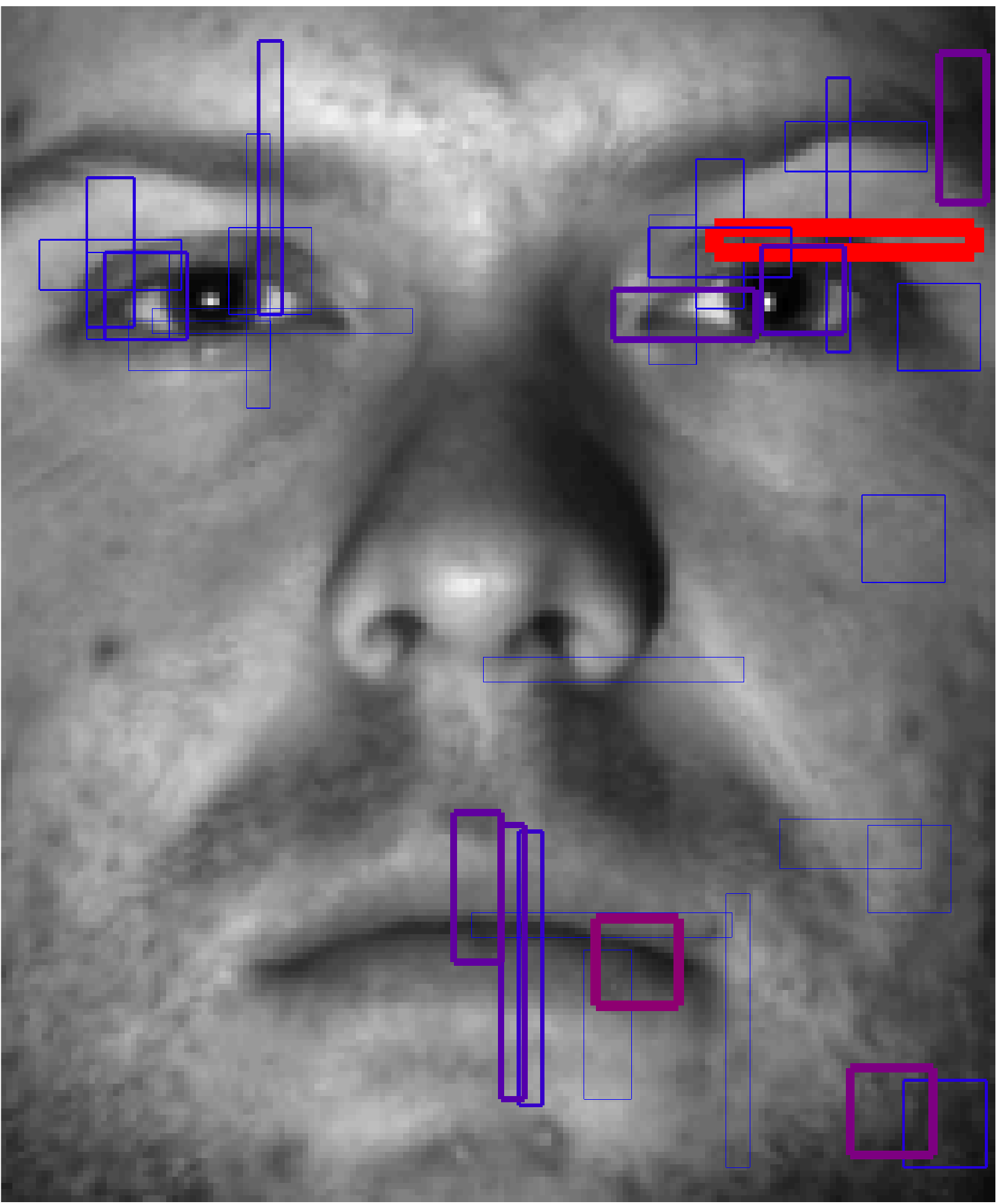}}
    \caption
    {
        The patch candidates (a) and those selected by ORE-Learning (b) and ORE-Boosting
        (c). All the patches are shown as blocks. Their widths and colors indicate the
        associated weights $\alpha_i,~\forall i$. A thicker and redder edge stands for a
        larger $\alpha_i$, \ie a more important patch. The ORE algorithms are conducted on
        a $100$-D feature space.
    }
  \label{fig:yaleb_selected_patches}
  \end{figure}

\subsection{Face recognition with random occlusions}
\label{subsec:infer_occlusion}

  The above task is completed nearly perfectly. However, sometimes the faces are
  contaminated by occlusions and most state-of-the-arts may fail on some of them. The most
  occlusions occur on face images could be divided into two categories: noisy occlusions
  and disguises. Let us consider the noisy ones first. The noisy occlusions are
  the ones not supposed to arise on a human face, or in other words, not face-related.
  They are unpredictable, and thus hard to learn. We then design a experiment to verify
  the inference capability of ORE-based methods. Considering that ORE-Learning and
  ORE-Boosting select similar patches, we only perform the former one in this test.
  To generate the corrupted samples for testing, we impose several Gaussian noise blocks on
  the Yale-B faces. The blocks are square and in the size of $s \times s,~s\in \{20, 40,
  60, 80, 100, 120\}$. The number of the blocks are defined by
  \begin{equation}
    N_{o} = \text{max}\{\text{round}(0.4 \sigma_f / s^2), 3\}, 
    \label{equ:number_blocks}
  \end{equation}
  where $\sigma_f$ represents the area of the whole face image. That is to say, the
  occluded parts won't cover more than $40\%$ area of the original face, unless the number
  requirement $N_{o} \ge 3$ is not met. The yielded faces are shown in
  Figure~\ref{fig:occluded_yaleb}. We can see that when $s = 120$, the contaminated parts
  dominate the face image.

  \begin{figure}[ht!]
    \centering
    \subfigure[$s = 20$]{\label{subfig:occlude_20}\includegraphics[width=0.2\textwidth]{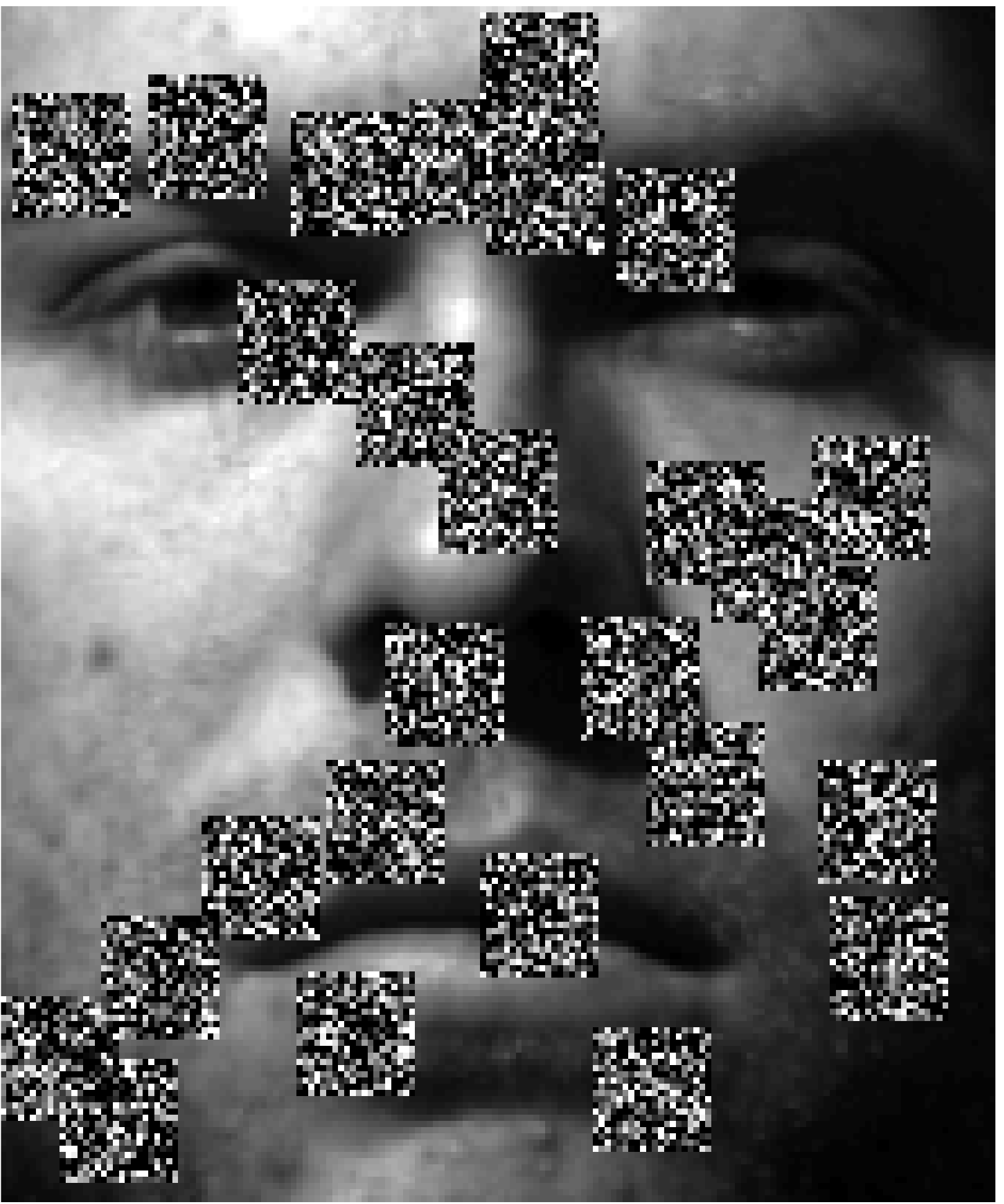}}
    \subfigure[$s = 40$]{\label{subfig:occlude_40}\includegraphics[width=0.2\textwidth]{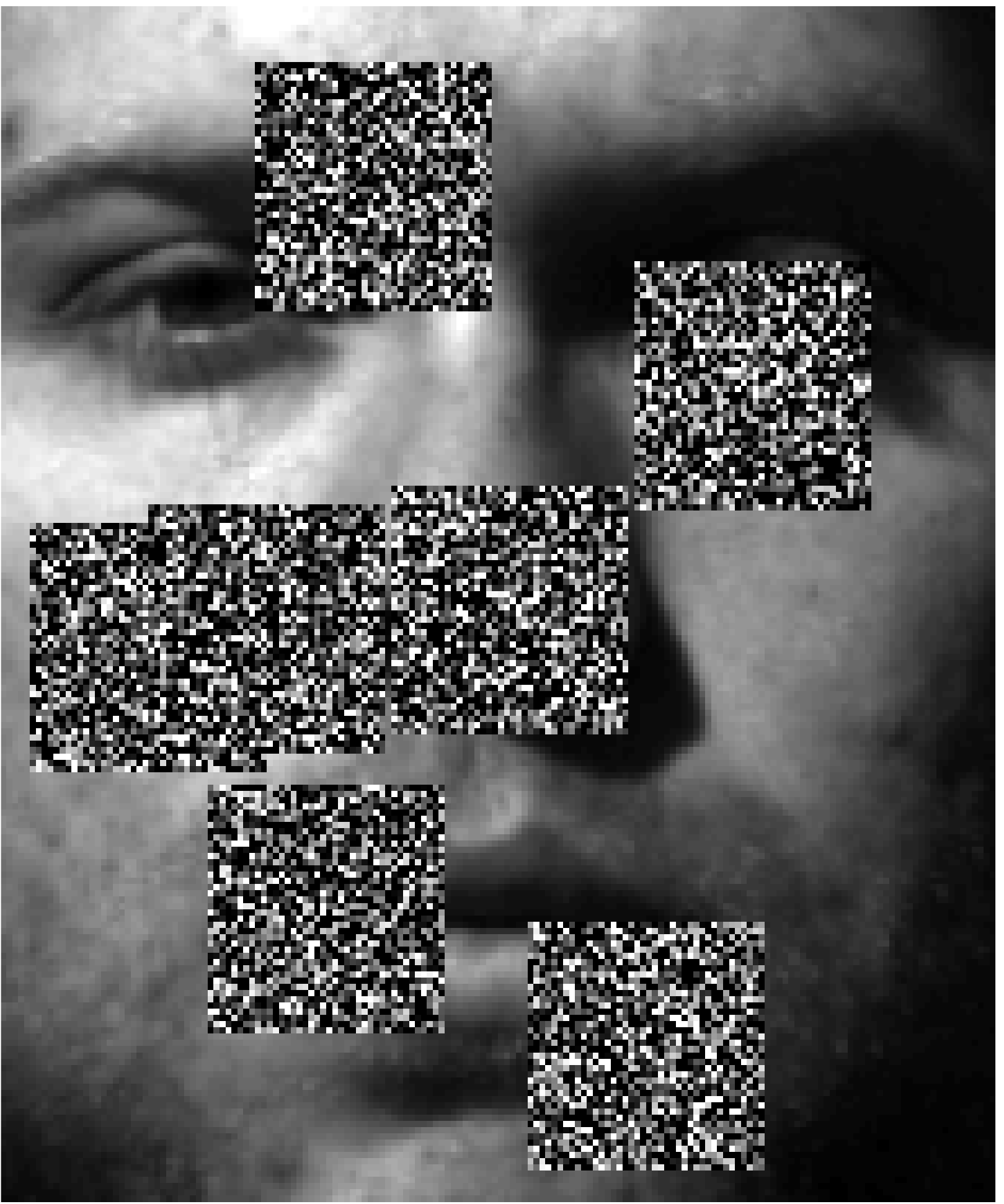}}
    \subfigure[$s = 80$]{\label{subfig:occlude_80}\includegraphics[width=0.2\textwidth]{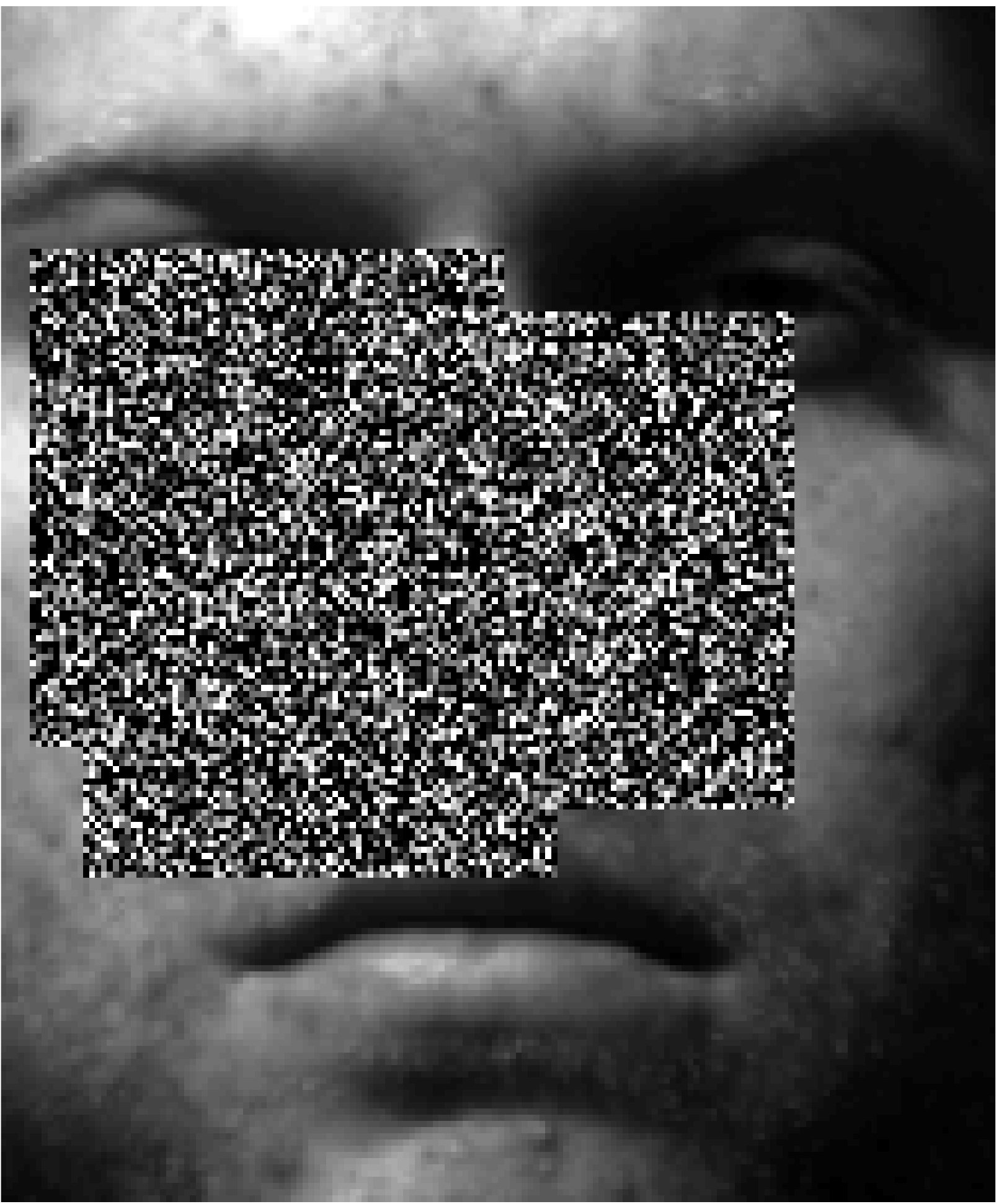}}
    \subfigure[$s = 120$]{\label{subfig:occlude_120}\includegraphics[width=0.2\textwidth]{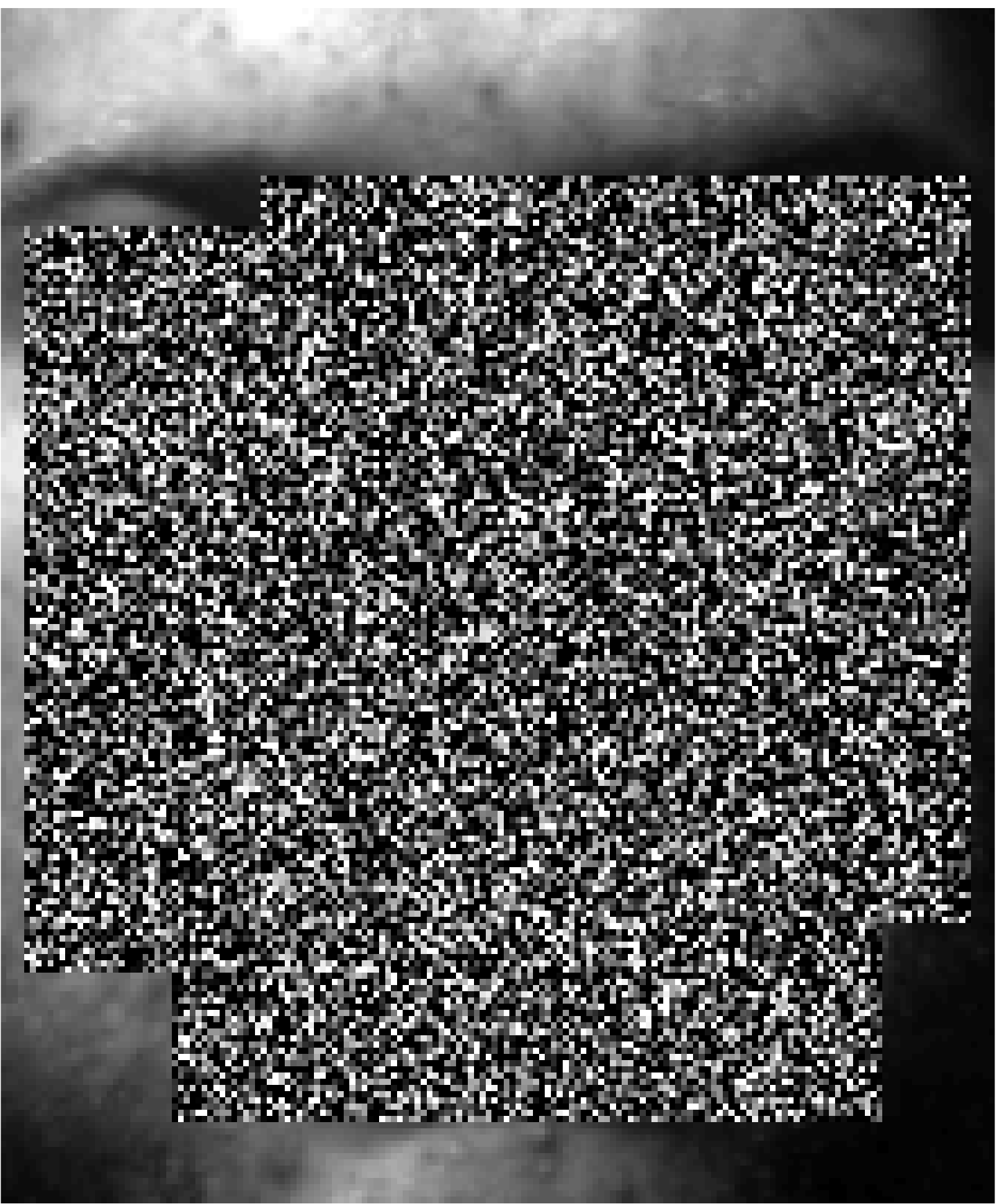}}
    \caption
    {
        The Yale-B faces with Gaussian noise occlusions. The block size is increased from
        $20$ to $120$. We can see that when $s = 120$, more than $60\%$ of the face image
        are totally contaminated.
    }
  \label{fig:occluded_yaleb}
  \end{figure}

  Before testing, we train our ORE models on the clean faces ($30$ faces for each
  individual). Then, on the contaminated faces (also $30$ faces for each individual), we
  test the learned models, with or without Robust-BPRs, comparing to the modular
  heuristics. In this way, we guarantee that no occlusion information is given in the
  training phase. As a reference, we also perform the standard LRC to illustrate different
  difficulty levels. The experiment is repeated $5$ times with the training and test faces
  selected randomly\footnote{We guarantee that a clean face and its contaminated version
  won't be selected simultaneously in each test.}. The results are shown in
  Table~\ref{tab:yaleb_occlude}. 
  
    \begin{table*}[ht!]
    \centering
    \resizebox{0.8\textwidth}{!}
    {
      \begin{tabular}{ l | l | l | l | l | l | l }

      \hline\hline
          & $s = 20$    & $s = 40$   & $s = 60$    & $s = 80$  & $s = 100$ & $s = 120$\\
      \cline{1-7}
      LRC ($400$-D) & $74.1\pm1.4$ &$69.7\pm1.3$  &$68.4\pm1.5$  & $45.5\pm1.4$ &
                    $30.4\pm0.7$ & $16.7\pm0.2$\\
      DEF         & $42.9\pm0.3$ & $80.1\pm0.4$ & $88.8\pm1.0$ & $72.3\pm0.6$ &
                  $48.0\pm1.4$ & $26.6\pm1.3$ \\
      Block-SRC   & $94.1\pm0.5$ & $93.3\pm0.5$ & $94.1\pm0.5$ & $85.7\pm0.8$ &
                  $78.3\pm0.4$ & $56.8\pm0.6$ \\
      \cline{1-7}
      ORE         & $93.9\pm2.6$ & $98.2\pm1.0$ & $98.8\pm0.6$ & $97.5\pm1.7$ &
      $94.2\pm3.6$ & $86.1\pm8.9$ \\
      Robust-ORE  & $\bf 98.5\pm0.7$ & $\bf 99.6\pm0.2$ & $\bf 99.7\pm0.1$ & $\bf
      99.4\pm0.5$ & $\bf 98.3\pm1.0$ & $\bf 93.8\pm4.6$ \\
      \hline\hline
      \end{tabular}
    }
    \caption
    {
        The comparison of accuracy on the occluded Yale-B. The highest recognition rates
        are shown in bold. Robust-ORE represents the ORE-model with Robust-BPRs. Note that
        the original LRC is performed with $400$-D Randomfaces.
    }
    \label{tab:yaleb_occlude}
    \end{table*}


  As we can see, again, the proposed ORE models achieve overwhelming performances. In
  particular, the original ORE-models are nearly (except for the case where $s = 20$)
  consistently better than all the state-of-the-arts.  Furthermore, the Robust-ORE models
  illustrate a very high robustness to the noisy occlusions. It is always ranked first in
  all the conditions and achieves the recognition rates above $98\%$ when $s < 120$.
  Recall that the performance obtained by ORE models on clean test sets is $99.9\%$.
  \emph{The severe occlusions merely reduce the performance of ORE model by around two
  percent.} When the face is dominated by continuous occlusions ($s = 120$), the
  accuracies of modular methods drop sharply to the ones below $60\%$ while that of
  Robust-ORE is still above $90\%$. This success justifies our assumption about the
  generic-face-patch linear-subspace. 

\subsection{Face recognition with expressions and disguises}
\label{subsec:infer_disguise}

  Another kind of common occlusions are functional disguises such as sunglasses and
  scarves. They are, generally speaking, face-related and intentionally put onto the
  faces. This kind of occlusions are unavoidable in real life. Besides this difficulty,
  expression is another important influential factor. Expressions invalidate the rigidity
  of the face surface, which is one foundation of the linear-subspace assumption. To
  verify the efficacy of our algorithms on the disguises and expressions, we
  employ the AR dataset. There are $100$ individuals in the AR (cropped version) dataset.
  Each subject consists of $26$ face images which come with different expressions and
  considerable disguises such as scarf and sunglasses (see Figure~\ref{fig:AR}). 

  \begin{figure}[ht!]
    \centering
    \includegraphics[width=0.68\textwidth]{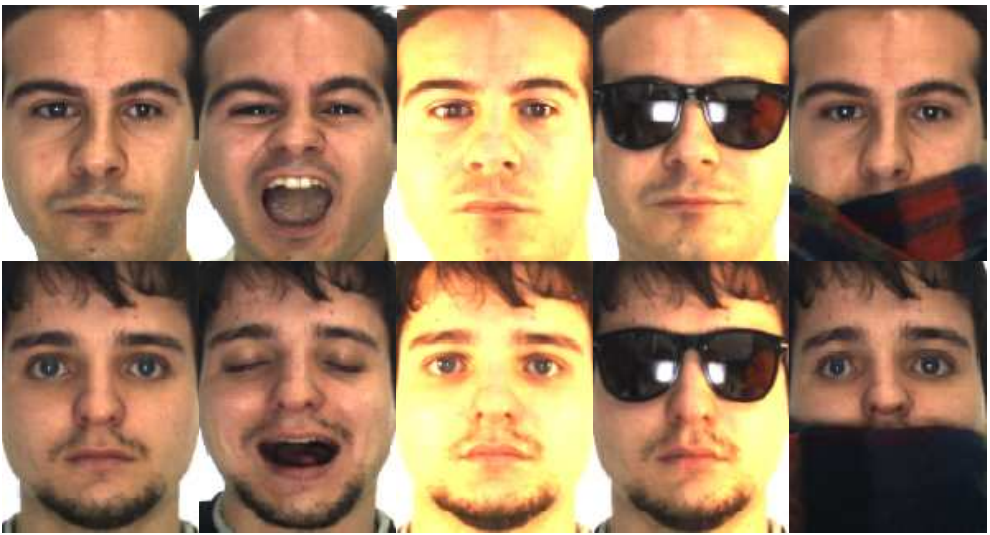}
    \caption
    {
        Images with occlusions and expressions in AR dataset. Note that we use only
        gray-scale faces in the experiment.
    }
    \label{fig:AR}
  \end{figure}

  First of all, following the conventional scheme, we use all the clean and inexpressive
  faces ($8$ faces for each individual) as training samples and test the algorithms on
  those with expressions ($6$ faces per individual), sunglasses ($6$ faces per individual)
  and scarves ($6$ faces per individual) respectively. Similar to the strategy for
  handling random occlusions, $500$ random patches are generated and we also employ the
  Robust-BPRs for testing. The test results can be found in Table~\ref{tab:ar_infer}. Note
  that the tests for Block-SRC, DEF and LRC are conducted once as the data split is
  deterministic. Consequently, no standard deviation is reported for those algorithms. The
  ORE-based methods are still run for $5$ times, with different random patches and random
  projections. 

    \begin{table}[ht!]
    \centering
    \resizebox{0.58\textwidth}{!}
    {
      \begin{tabular}{ l | c | c | c }

      \hline\hline
        & Expressions & Sunglasses & Scarves \\
      \cline{1-4}
      LRC ($400$-D) & $81.0$ & $54.5$ & $10.7$ \\
      DEF           & $88.2$ & $91.2$ & $85.2$ \\
      Block-SRC     & $87.5$ & $95.7$ & $86.0$\\
      \cline{1-4}
      ORE           & $82.0 \pm 1.2$ & $85.0 \pm 3.9$ & $86.5 \pm 0.7$ \\
      Robust-ORE    & $\bf 92.8 \pm 0.9$ & $\bf 96.1 \pm 1.8$ & $\bf 95.8 \pm 1.2$ \\
      \hline\hline
      \end{tabular}
    }
    \caption
    {
        The comparison of accuracy on the AR dataset. The highest recognition rates
        are shown in bold. Robust-ORE represents the ORE-model with Robust-BPRs. Note that
        the original LRC is performed with $400$-D Randomfaces.
    }
    \label{tab:ar_infer}
    \end{table}


  According to the table, Robust-ORE beats other methods in all the scenarios. In
  particular, for the faces with scarves, both of ORE and Robust-ORE are superior to other
  methods. The performance gap between Robust-ORE and the involved state-of-the-arts is
  around $10\%$. 

\subsection{Learn the patterns of disguises and expressions}
\label{subsec:learn_disguise}

  The expressions and disguises share one desirable property: they can be characterized by
  typical and limited patterns. One thus can learn those patterns within our ensemble
  learning framework. To verify the learning power of the proposed method, we re-split the
  data: for each individual, $13$ images are randomly selected for training while the
  remaining ones are test images. In this way, the ORE-Learning or ORE-Boosting algorithm
  is given the information on disguise patterns. The experiment on AR is rerun in the new
  setting. Table~\ref{tab:ar_learn} shows the recognition accuracies. Note that the
  results for the $100$-D Fisherface are actually obtained by using $95$-D features since
  here ($100$ categories) the dimensionality limit for LDA is $99$.

    \begin{table*}[ht]
    \centering
    \resizebox{0.7\textwidth}{!}
    {
      \begin{tabular}{ l | l | l | l | l | l | l }

      \hline\hline
      &   & $25$-D    & $50$-D   & $100$-D    & $200$-D    & $400$-D   \\
      \cline{1-7}
      {\multirow{4}{*}{LDA}}  
      & NN        & $95.3\pm0.3$ & $97.4\pm0.6$ & $97.9\pm0.5$ & - & - \\
      & NFL       & $92.5\pm0.8$ & $96.8\pm0.4$ & $97.9\pm0.2$ & - & - \\
      & SRC       & $94.6\pm0.5$ & $97.4\pm0.5$ & $97.9\pm0.4$ & - & - \\
      & LRC       & $72.8\pm1.6$ & $94.5\pm0.4$ & $97.1\pm0.3$ & - & - \\
      \cline{1-7}
      {\multirow{4}{*}{Rand}}  
      & NN        & $17.0\pm0.9$ & $19.8\pm1.6$ & $22.8\pm2.3$ & $22.2\pm1.5$ & $23.6\pm1.1$ \\
      & NFL       & $44.9\pm3.0$ & $55.2\pm1.9$ & $60.9\pm1.7$ & $63.1\pm1.2$ & $65.1\pm1.2$ \\
      & SRC       & $45.4\pm0.5$ & $71.3\pm1.7$ & $85.8\pm1.1$ & $91.5\pm0.7$ & $93.9\pm0.6$ \\
      & LRC       & $43.0\pm2.2$ & $71.6\pm1.8$ & $78.9\pm1.3$ & $82.1\pm1.2$ & $83.5\pm0.7$ \\
      \cline{1-7}
      {\multirow{4}{*}{PCA}}  
      & NN        & $19.4\pm1.3$ & $20.4\pm1.1$ & $21.7\pm1.3$ & $21.8\pm1.2$ & $22.0\pm1.0$ \\
      & NFL       & $41.9\pm1.6$ & $48.2\pm1.2$ & $52.1\pm1.6$ & $54.3\pm1.3$ & $55.4\pm1.3$ \\
      & SRC       & $52.7\pm0.8$ & $72.1\pm1.5$ & $80.8\pm1.0$ & $83.6\pm0.5$ & $83.9\pm0.7$ \\
      & LRC       & $60.3\pm0.8$ & $75.3\pm1.0$ & $80.3\pm0.7$ & $82.1\pm0.8$ & $82.7\pm0.8$ \\
      \cline{1-7}
      \multicolumn{2}{l|}{{ORE}}        
      & $97.0\pm0.5$ & $98.7\pm0.5$ & $99.0\pm0.3$ & $99.1\pm0.1$ & - \\
      \multicolumn{2}{l|}{{Robust-ORE}}  
      & $\bf 98.4\pm0.5$ & $\bf 99.1\pm0.4$ & $\bf 99.4\pm0.2$ & $\bf 99.5\pm0.2$ & - \\
      \multicolumn{2}{l|}{{Boosted-ORE}}  
      & $96.8\pm0.3$ & $98.6\pm0.3$ & $98.9\pm0.3$ & $99.0\pm0.4$ & - \\
      \hline\hline
      \end{tabular}
    }
    \caption
    {
        The comparison of accuracy on AR. The highest recognition rates are shown in bold.
        Note that we only perform algorithms with the Fisherface (LDA) on the $25$-D and
        $50$-D feature spaces. The original patch has $225$ pixels, thus we can't conduct
        ORE algorithms with $400$-D features.
    }
    \label{tab:ar_learn}
    \end{table*}


  Similar to the previous test, our methods once again show overwhelming superiority. {\em
  The Robust-ORE algorithm achieves a recognition rate of $99.5\%$ which is also the
  best reported result on AR in the similar experimental setting.}. In this sense, we can conclude
  that {\em the ORE algorithms can effectively learn the patterns of disguises}. The
  boosting-like variation of ORE-Learning obtains remarkable performances as well, but is
  slightly worse than the original version. Besides the ORE algorithms, the Fisherface
  approach also shows a high learning capacity. With Fisherfaces, the simplest Nearest
  Neighbor algorithm already achieves the recognition rate of $97.9\%$. This empirical
  evidence implies that discriminative face recognition methods usually benefit from
  learning certain face-related patterns.



  Figure~\ref{fig:ar_selected_patches} shows the patch candidates
  (Figure~\ref{subfig:ar_all_patches}) and the selected ones for ORE-Learning
  (Figure~\ref{subfig:ar_ore_patches}) and ORE-Boosting
  (Figure~\ref{subfig:ar_boost_patches}). As illustrated in the figure, the $500$ patch
  candidates redundantly samples the face image. Both ORE-Learning and ORE-Boosting choose
  $54$ patches and ORE-Learning still employs a more conservative strategy of weight
  assignment. Differing from Figure~\ref{fig:yaleb_selected_patches}, the ORE algorithms
  now focus on the forehead more than eyes and the mouth. Considering that sunglasses and
  scarves are usually located in those two places, the disguises' patterns are learned
  and the corresponding patch positions are less trusted during the test.

  \begin{figure}[ht!]
    \centering
    \subfigure[candidates]{\label{subfig:ar_all_patches}\includegraphics[width=0.25\textwidth]{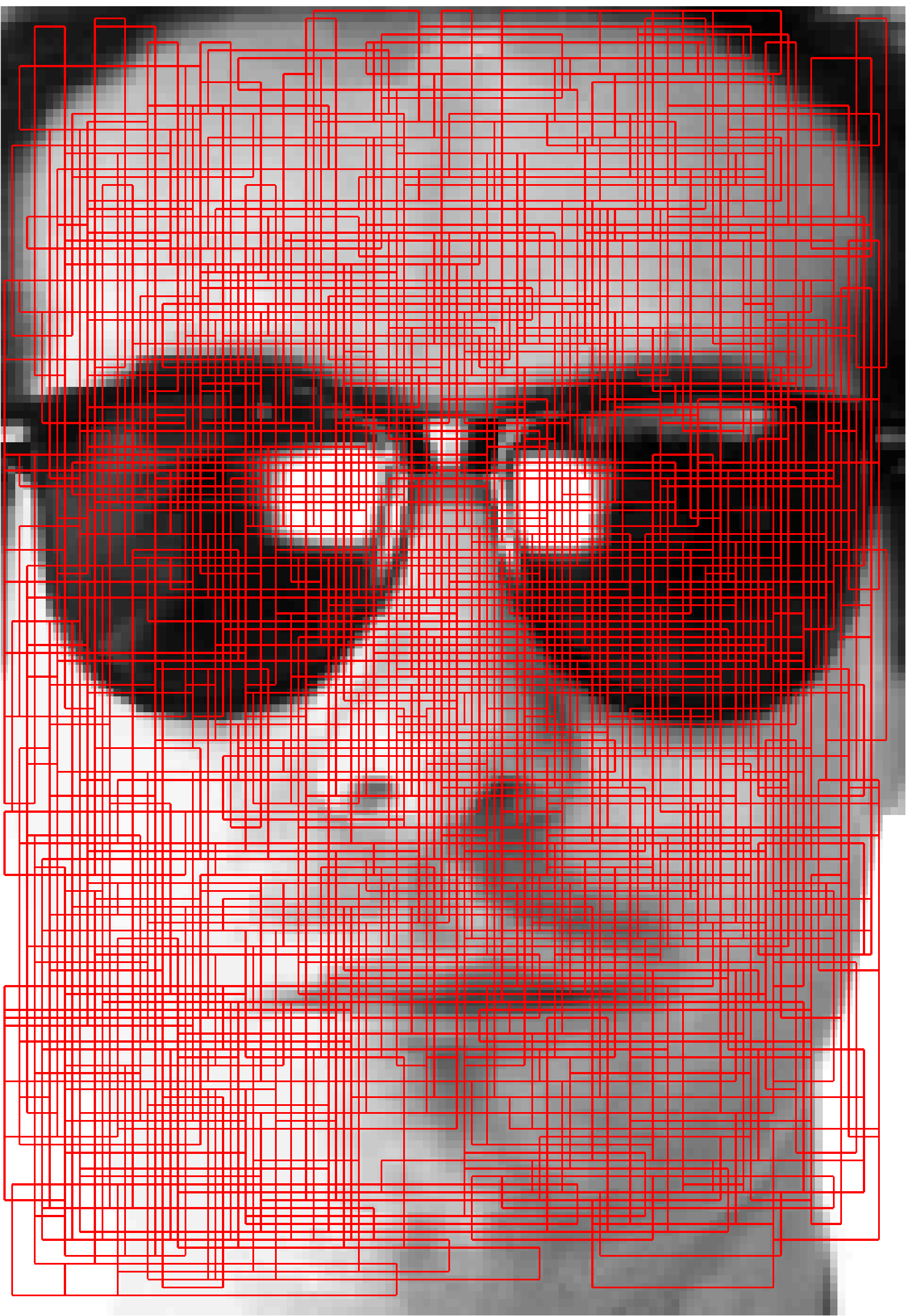}}
    \subfigure[ORE-Learning]{\label{subfig:ar_ore_patches}\includegraphics[width=0.25\textwidth]{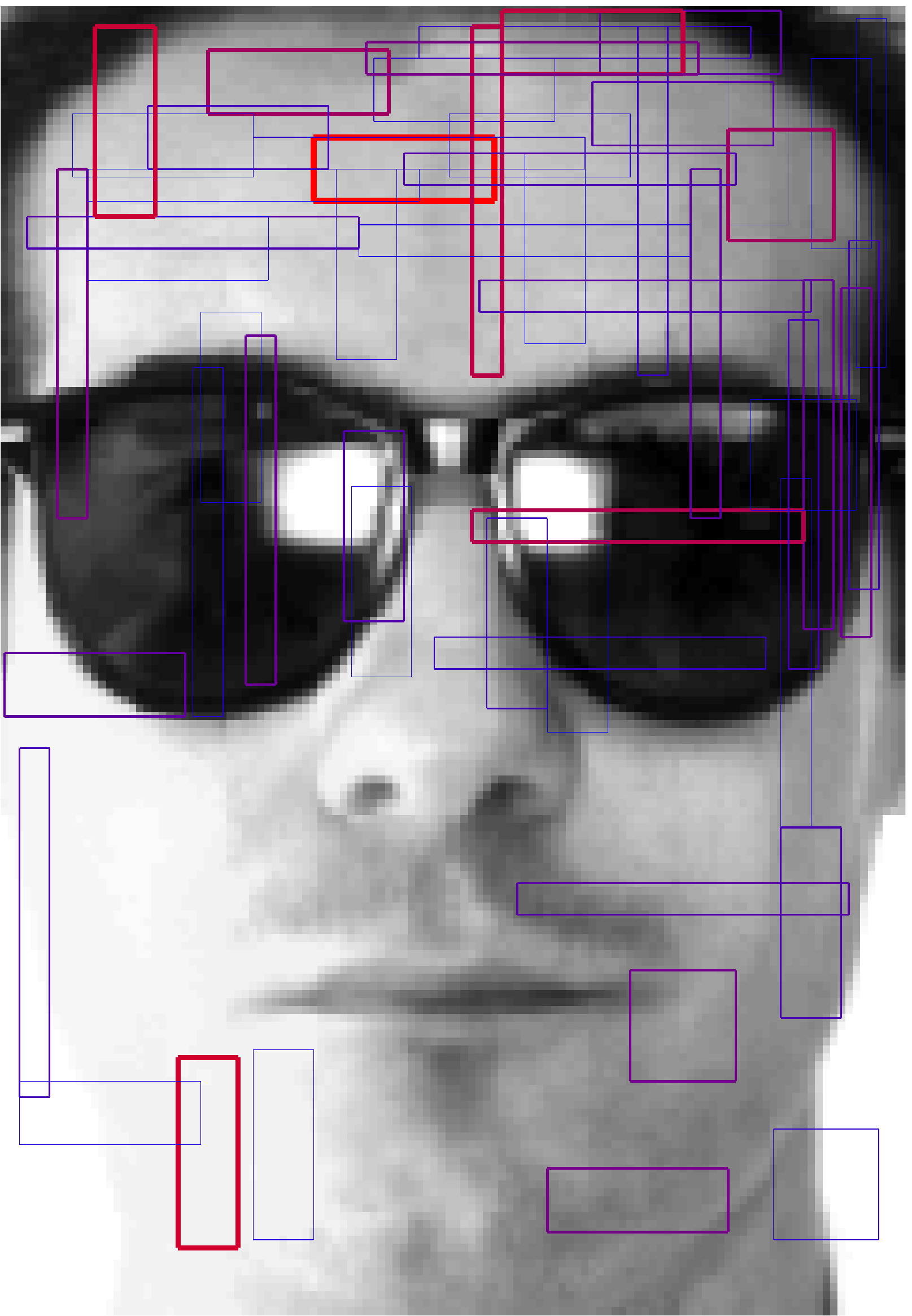}}
    \subfigure[ORE-Boosting]{\label{subfig:ar_boost_patches}\includegraphics[width=0.25\textwidth]{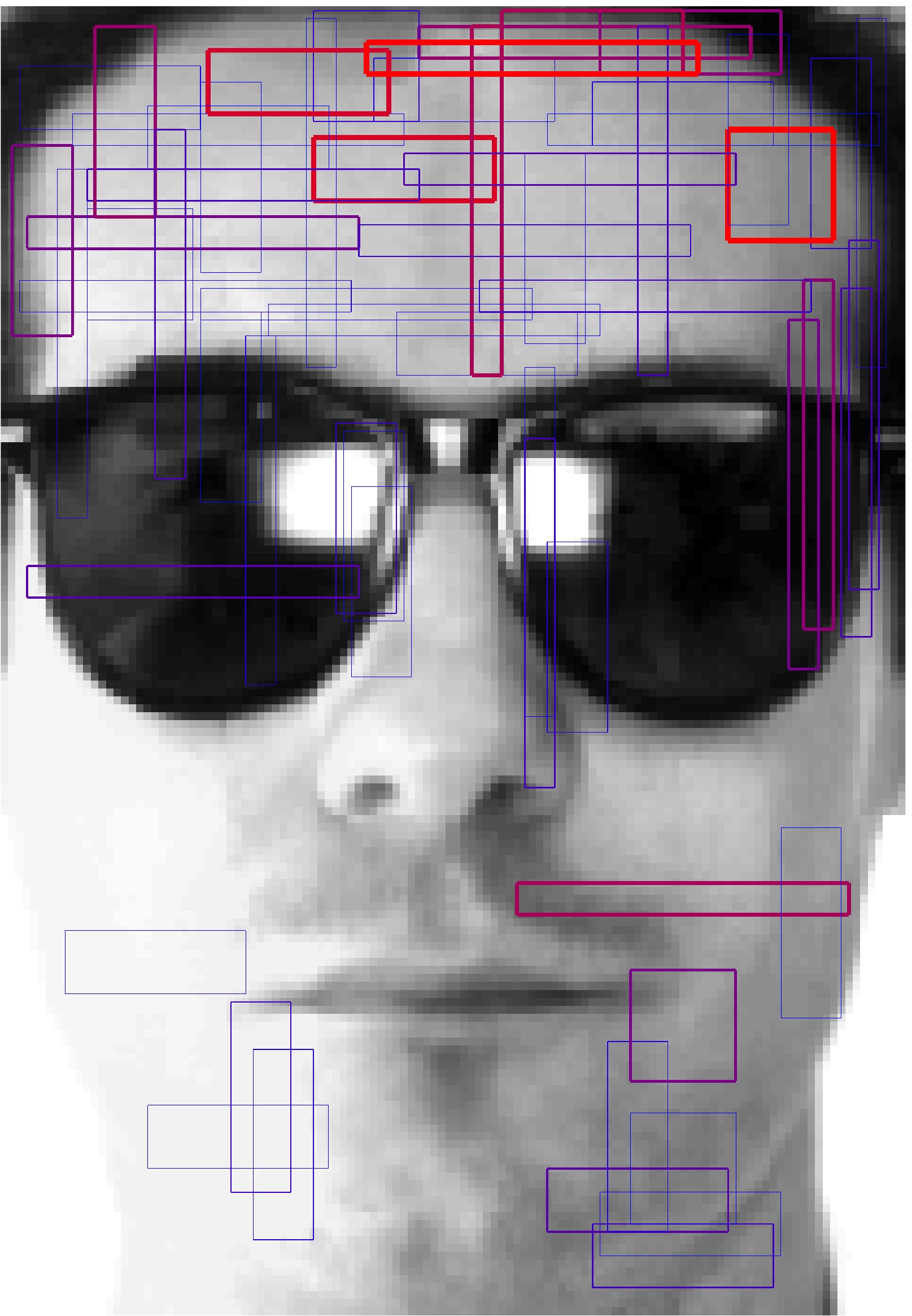}}
    \caption
    {
        The patch candidates (a) and those selected by ORE-Learning (b) and ORE-Boosting
        (c). All the patches are shown as blocks. Their widths and colors indicate the
        associated weights $\alpha_i,~\forall i$. A thicker and redder edge stands for a
        larger $\alpha_i$, \ie a more important patch. The ORE algorithms are conducted on
        a $100$-D feature space.
    }
  \label{fig:ar_selected_patches}
  \end{figure}

\subsection{Efficiency}
\label{subsec:exp_efficiency}
  
  For a practical computer vision algorithm, the running speed is usually crucial. Here we
  show the extremely high efficiency of the proposed algorithms, in both terms of training
  and test.

\subsubsection{The verification for the fast model selection}

  First of all, let us verify the training-determined model-selection for ORE-Learning and
  ORE-Boosting. Figure~\ref{fig:train_test_errors_l1} demonstrates the training error and
  test error curves as the model complexity, factorized by $1/\lambda$, is increasing. The
  two curves, as we can see, show nearly identical tendencies. In particular, when $1/\lambda
  = 1e$-$5$, \ie only trivial regularization is imposed, we still can not observe any
  deviation between the two errors. In other words, overfitting does not occur.
  Consequently, we can employ the training error of ORE-Learning as an accurate
  measurement of the generalization ability and chose an proper model parameter on it. The
  required time for model-selection is therefore reduced dramatically. 
  
  \begin{figure}[ht!]
    \centering
    \includegraphics[width=0.68\textwidth]{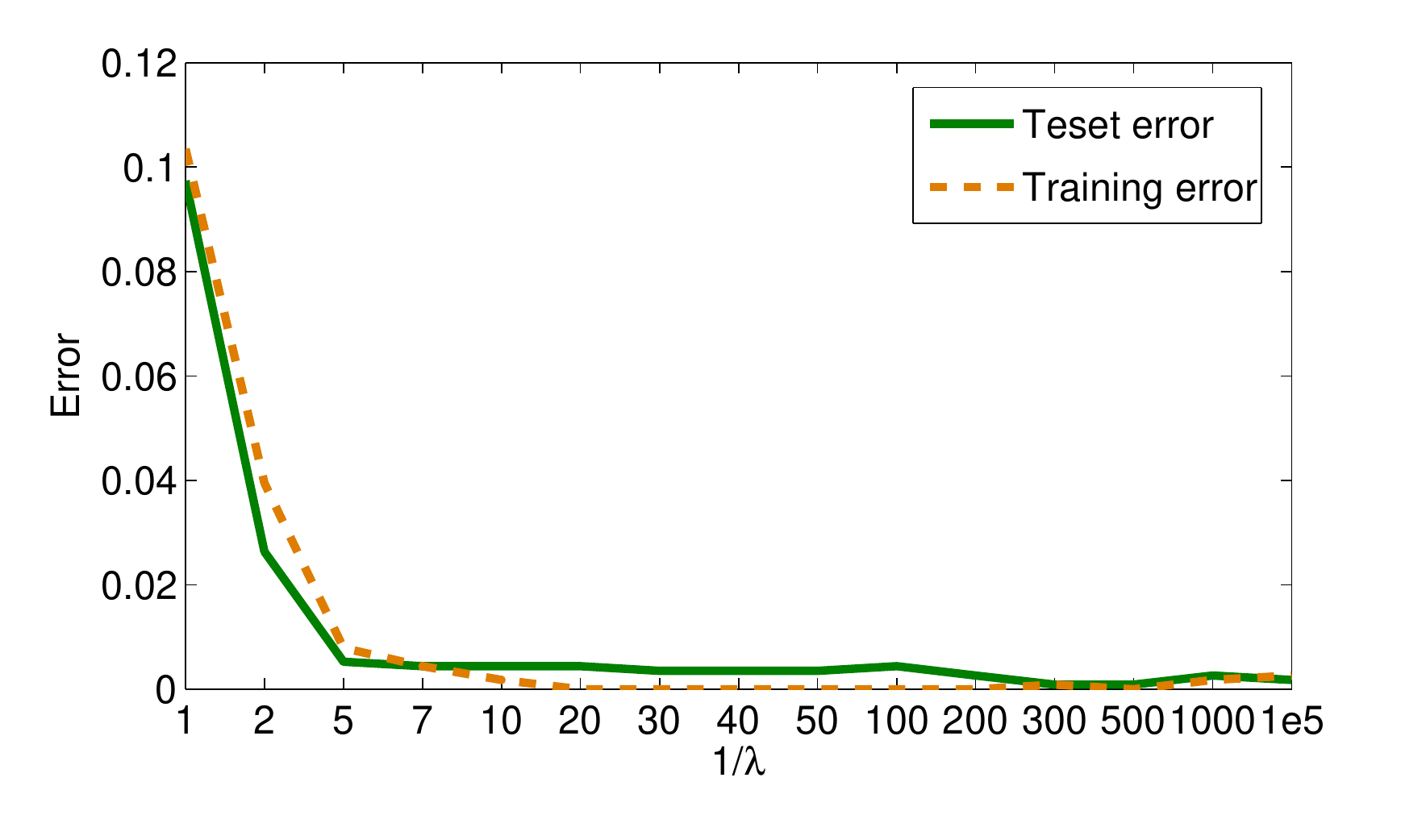}
    \caption
    {
        The training errors and test errors change with the increasing model complexity:
        $\frac{1}{\lambda}$. Note that the $x$-axis is not linearly scaled. The results
        are generated in a $100$-D feature space with Yale-B faces.
    }
  \label{fig:train_test_errors_l1}
  \end{figure}

\subsubsection{The improvement on the training speed}
\label{subsec:fast_training}

  Besides the fast model selection, we also theoretically validate the fast training
  procedure. It achieves, as illustrated above, very promising accuracies. Here we
  evaluate the improvement on the training speed. Figure~\ref{fig:fast_training} depicts
  the difference on the time consumptions for training a Boosted-ORE model, between the methods
  with and without updating BPRs at every iteration. The test is conducted with the
  increasing number (from $10$ to $2,000$) of BPRs and trade-off parameter $\lambda = 0.02$
  in the $100$-D feature space. As illustrated, the efficiency gap is huge. Without the
  BPR-recalculation, one could save the training time by from $700$ seconds ($10$ BPRs) to
  more than $10$ days ($2,000$ BPRs).

  \begin{figure}[ht!]
    \centering
    \includegraphics[width=0.68\textwidth]{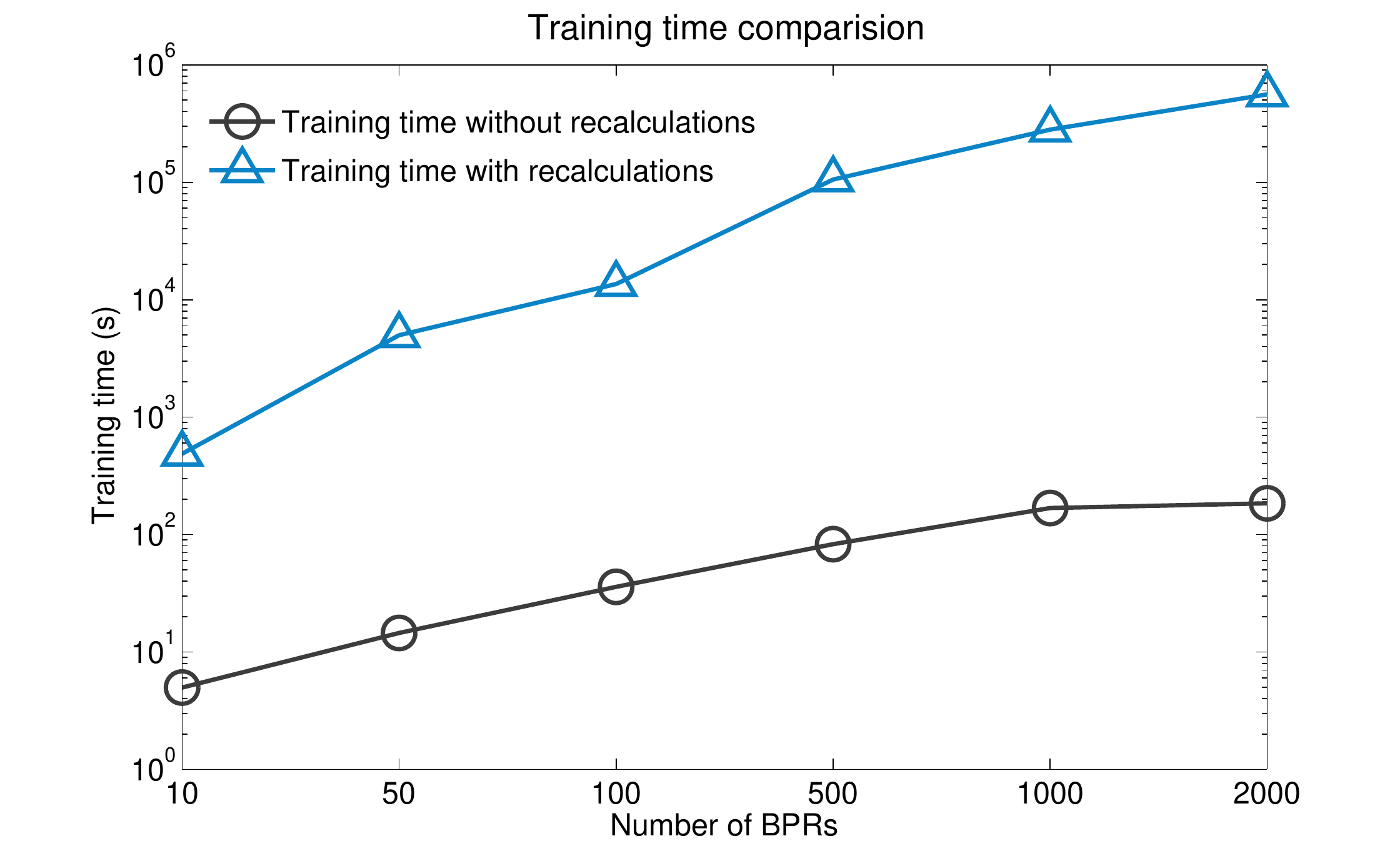}
    \caption
    {
        The training times consumed by the ORE-Boosting methods with the BPR
        recalculations and without them. Note that the $y$-axis is shown in the
        logarithmic scale. The results are obtained on the AR dataset with $100$-D
        features. 
    }
  \label{fig:fast_training}
  \end{figure}

\subsubsection{The highest execution efficiency}

  At last, let us verify the most important efficiency property --- execution speed. The
  test face (or face patch) is randomly mapped to a lower-dimensional space. Given a
  reduced dimensionality, all the face recognition algorithms are performed $100$ times on
  faces from Yale-B. We record the elapsed times (in ms) for each method and show the
  average values in Figure~\ref{fig:efficiency}. Note that for LRC and ORE-based methods,
  there is no need to perform LRs when testing as all the representation bases are 
  deterministic. Before test, one can pre-calculate and store all the matrices   
  \begin{equation}
    \E = (\hat{\X}^\T \hat{\X})^{-1}{\hat{\X}}^\T,
    \label{equ:fast_test}
  \end{equation}
  where $\hat{\X}$ represents different basis for different algorithms. Then the
  representation coefficients $\Beta$ for the test face (or patch) $\y$ can be obtained
  via a simple matrix multiplication, \ie
  \begin{equation}
    \Beta = \E\y.
    \label{equ:simple_multiple}
  \end{equation}
  %
  
  \begin{figure*}[ht]
    \centering
    \includegraphics[width=1\textwidth]{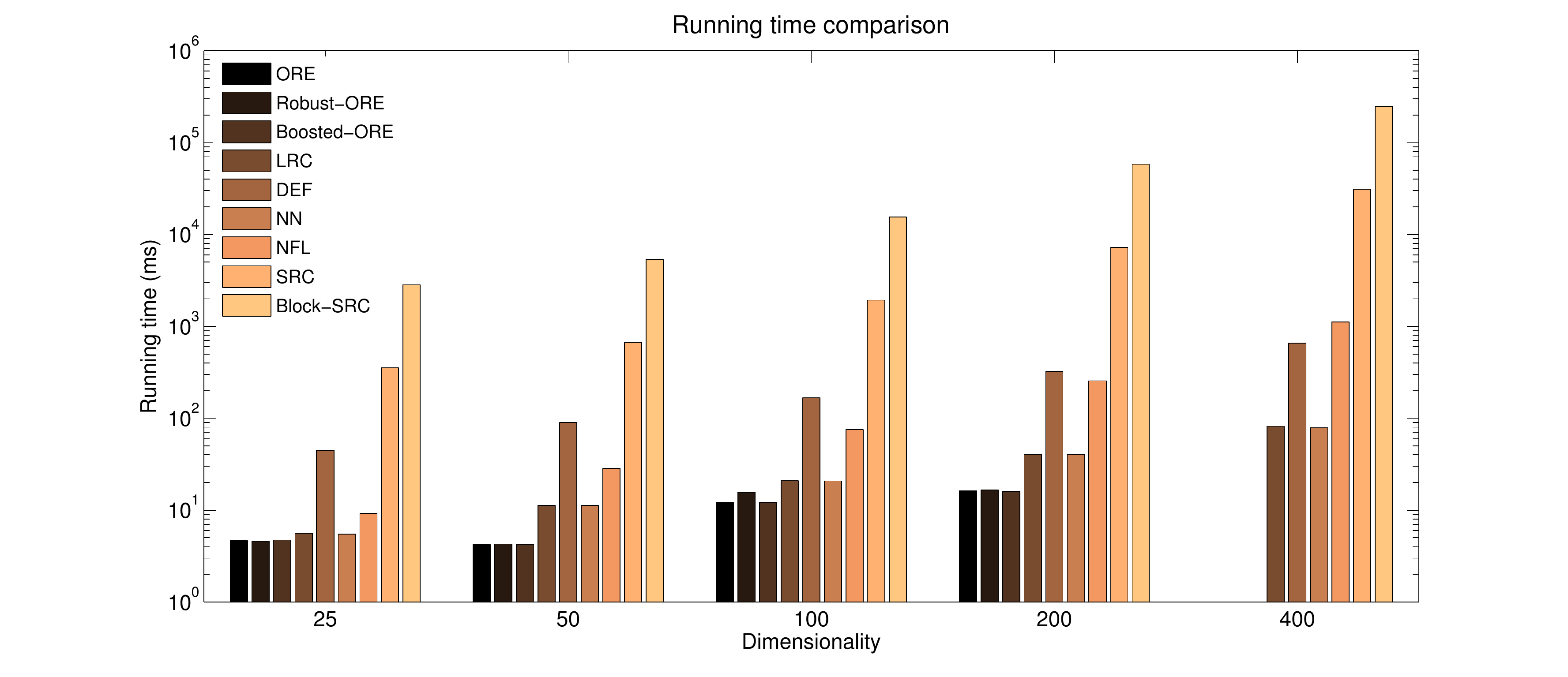}
    \caption[Comparison of running time for face recognition algorithms]
    {
      Comparison of the running time. Note that the $y$-axis is in the logarithmic scale.
      We don't perform ORE algorithms in the $400$-D feature space as each patch only has
      $225$ pixels. The $200$-D results for the proposed methods are actually obtained in
      the original $225$-D space.
    }
    \label{fig:efficiency}
  \end{figure*}

  As demonstrated, the SRC-based algorithms are the slowest two. The original SRC needs up
  to $31$ second ($400$-D) to process one test face. The Block-SRC approach, which shows
  relatively high robustness in the literature, shows even worse efficiency.  For $400$-D
  features, one need to wait more than $4$ minutes for one prediction yielded by
  Block-SRC. NFL also performs slowly. It requires $9$ to $1,113$ ms to handle one test
  image. In contrast, the ORE-based methods consistently outperform others in terms of
  efficiency. In particular, on the $200$-D ($225$-D in fact) feature space, one only
  needs $16$ ms to identify a probe face by using either ORE algorithm. This speed not
  only overwhelms those of SRC and NFL, but is also $2$-time higher than those of LRC and
  NN.
  
  Such a high efficiency, however, seems not reliable. Intuitively, the time consumed by
  LRC might be always shorter than that for ORE because ORE performs multiple LRs (here
  actually matrix multiplications) while LRC only performs one. We then track the
  execution time of the Matlab code via the ``profile'' facility. We found that, with
  high-dimensional features and efficient classifiers, it is the dimension reduction which
  dominates the time usage. The NN and LRC algorithms both perform the linear projection
  over all the pixels while ORE only select a small part (the pixels in the patches) of
  them to do the dimension-reduction.  As a result, if not too many patches are selected,
  the ORE algorithms usually illustrate even higher efficiencies than LRC and NN. 

  Recall that the proposed methods achieve almost all the best recognition rates in
  different conditions. We draw the conclusion that, \emph{the ORE model is a very
  promising face recognizer which is not only most accurate, but also most efficient.}

\section{Conclusion and future topics}
\label{sec:conclusion}

  In this paper, a learning-inference-mixed framework is proposed for face recognition.
  By observing that, in practice, only partial face is reliable for the linear-subspace
  assumption. We generate random face patches and conduct LRs on each of them. The
  patch-based linear representations are interpreted by using the Bayesian theory and
  linearly aggregated via minimizing the empirical risks. The yielded combination, Optimal
  Representation Ensemble, shows high capability of learning face-related patterns and
  outperforms state-of-the-arts on both accuracy and efficiency. With ORE-models, one can
  almost perfectly recognize the faces in Yale-B (with the accuracy $99.9\%$) and AR (with
  the accuracy $99.5\%$) dataset, and at a remarkable speed (below $20$ ms per face using
  the unoptimized Matlab code and one CPU core). 
  
  For handling foreign patterns arising in test faces, the Generic-Face-Confidence is
  derived by taking the non-face patch into consideration. Facilitated by GFCs, the
  ORE-model shows a high robustness to noisy occlusions, expresses and disguises. It beats
  the modular heuristics under nearly all the circumstances. In particular, for Gaussian
  noise blocks, the recognition rate of our method is always above $93\%$ and fluctuates
  around $99\%$ when the blocks are not too large. For real-life disguises and facial
  expressions, Robust-ORE also outperforms the competitors consistently.

  In addition, to accommodate the instance-based BPRs, an novel ensemble learning
  algorithm is designed based on the proposed leave-one-out margins. The learning
  algorithm, ORE-Learning, is theoretically and empirically proved to be resistant to
  overfittings. This desirable property leads to a training-determined model-selection,
  which is much faster than conventional $n$-fold cross-validations. For immense BPR sets,
  we propose the ORE-Boosting algorithm to exploit the vast functional spaces.
  Furthermore, we also increase the training speed a lot by proving that the ORE-Boosting
  is actually data-weight-free. 
  
  As to the future work, one promising direction is to exploit the spatial information for
  ORE-models. Similar to \cite{Zhou_09_ICCV_MRF}, one could also employ a \emph{Markov
  Random Field} (MRF) method to analyze the patch-based GFCs.  Even higher accuracies
  could be achieved, considering that the GFC is more informative and robust than a single
  pixel. Secondly, ORE-models can be expanded for the video-based face recognition via
  using online-learning algorithms. Considering that the Robust-BPRs can distinguish face
  parts from the non-facial ones, we also want to design a ORE-based face detection
  algorithm. By merging the detector with this work, we could finally obtain a
  multiple-task ORE-model that performs the detection and recognition simultaneously. 
  
{
  \bibliographystyle{IEEEbib}
  \bibliography{tip_arc}
}

\end{document}